\documentclass[11pt]{article} 

\usepackage{jmlr2e_cus}
\usepackage[usenames]{color}
\usepackage[bookmarksnumbered=true]{hyperref}
\usepackage{url}
\definecolor{cite_color}{RGB}{0, 51, 153}
\definecolor{link_color}{RGB}{153, 0,0}  
\definecolor{url_color}{RGB}{153, 102,  0}
\definecolor{emp_color}{RGB}{0,0,255}
\hypersetup{
 colorlinks,
 citecolor=cite_color,
 linkcolor=link_color,
 urlcolor=url_color}

\usepackage{amsmath}
\usepackage{amssymb}
\usepackage{graphicx}
\usepackage{subfig}
\usepackage{multirow}
\usepackage{booktabs}
\usepackage{mathrsfs} 
\usepackage{wrapfig}
\usepackage{color}
\usepackage{enumerate}
\usepackage{bm}
\usepackage{bbm}
\usepackage{tabularx}
\usepackage{natbib}
\bibpunct{(}{)}{;}{a}{,}{,}
\usepackage{tikz}

\def\X{{\cal X}}

\def\A{{\mathscr A}}

\def \v{\bm{v}}

\def \b{\bm{b}}

\def \a{\bm{a}}
\def \x{\bm{x}}
\def \y{\bm{y}}

\def \bu{\bm{u}}  
\def \bmA{\mathbf{A}}

\def \bmH{\mathbf{H}}

\def \bmchi{\bm{\chi}}

\def \P{{\cal{P}}}

\def \R{{\mathbb{R}}}
\def \trans{\top}
\def \chara{\bm{\chi}} 

\def \h{\bm{h}}

\def \cg{L}

\newcommand{\opt}{\bm{x}^*}
\newcommand{\opti}[1]{x^*_{#1}} 
\newcommand{\groundset}{E}

\newcommand{\NEWDR}{\texttt{weak DR}}
\newcommand{\ALG}{Algorithm }
\newcommand{\SUPP}{{Appendix }}
\newcommand{\SEC}{{Section }}
\newcommand{\EQ}{Equation }
\newcommand{\FIG}{Figure }

\newcommand{\argmax}{{\arg\max}}
\newcommand{\algname}[1]{{\textsc{#1}}}

\newcommand{\spt}[1]{{\text{supp}}(#1)}
\newcommand{\dtp}[2]{\langle #1, #2\rangle}
\newcommand{\fracpartial}[2]{\frac{\partial #1}{\partial #2}}
\newcommand{\fracppartial}[2]{\frac{\partial^2 #1}{\partial #2}}

\newcommand{\sete}[3]{\bm #1|#1_{#2}\!\!\gets\!\! #3}   
\usepackage[vlined, ruled, linesnumbered]{algorithm2e}

\SetCommentSty{mycommfont}
\SetKwInOut{Input}{input}
\SetKwInOut{Output}{output}

\newtheorem{theorem}{Theorem}
\newtheorem{lemma}{Lemma}
\newtheorem{proposition}{Proposition}

\newtheorem{corollary}{Corollary}
\newtheorem{definition}{Definition}

\firstpageno{0}

\title{\LARGE Guaranteed Non-convex Optimization: \\
Submodular Maximization over Continuous Domains\thanks{Appears in the $20^{{th}}$ International Conference on Artificial Intelligence and Statistics (AISTATS) 2017, Fort Lauderdale, Florida, USA.}
}

\author{
Andrew	An Bian\thanks{{\color{blue}{Now known as Yatao A. Bian. ORCID: \href{https://orcid.org/0000-0002-2368-4084}{orcid.org/0000-0002-2368-4084}}}} \\
	ETH Z{u}rich \\
	\texttt{ybian@inf.ethz.ch}\\
	 \And
	Baharan Mirzasoleiman\\
	ETH Z{u}rich\\
	\texttt{baharanm@inf.ethz.ch} \\
	 \AND
	Joachim M. Buhmann\\
	ETH Z{u}rich\\
	\texttt{jbuhmann@inf.ethz.ch} \\
	 \And
	Andreas Krause\\
	ETH Z{u}rich\\
	\texttt{krausea@ethz.ch} \\
	\AND
\normalfont{\large December 9, 2016}
}

\begin{document}

\maketitle

\vspace{1cm}

 \begin{abstract}
 \noindent 
   \textit{Submodular continuous functions} are a category of
    (generally) non-convex/non-concave functions with a wide
   spectrum of applications. We characterize these functions and
   demonstrate that they can be maximized efficiently with
   approximation guarantees.
   Specifically, 
   i) We introduce  the \NEWDR\ property that gives a unified characterization of submodularity for all  set, integer-lattice and continuous functions; 
   ii) for maximizing monotone {DR-submodular} continuous functions under general  down-closed convex constraints, we propose a
   \algname{Frank-Wolfe} variant with $(1-1/e)$ approximation guarantee,
   and sub-linear convergence rate; 
   iii) for maximizing general \textit{non-monotone}
   submodular continuous functions subject to box constraints, we propose a
   \algname{DoubleGreedy} algorithm with
   $1/3$ approximation guarantee. 
   Submodular continuous functions naturally find
   applications in various real-world settings, including influence and revenue
   maximization with continuous assignments, sensor energy management,
multi-resolution data summarization, 
   facility location, 
    etc.  Experimental
   results show that the proposed algorithms efficiently generate
   superior solutions  compared to
   baseline algorithms.
\end{abstract}

\clearpage


\section{Introduction}

Non-convex optimization delineates the new frontier in machine
learning, arising in numerous learning tasks from training deep neural
networks to latent variable models \citep{anandkumar2014tensor}.  Understanding, which classes of
objectives can be tractably optimized  remains a central challenge. In this paper, we
investigate a class of generally non-convex and non-concave functions--{\it submodular continuous functions}, and derive algorithms for
approximately optimizing them with strong approximation
guarantees.

Submodularity is a structural property usually associated with {\it
  set functions}, with important implications for
optimization. Optimizing submodular set functions has found numerous
applications in machine learning, including variable selection
\citep{DBLP:conf/uai/KrauseG05}, dictionary learning
\citep{krause2010submodular,das2011submodular}, sparsity inducing regularizers
\citep{bach2010structured}, summarization
\citep{lin2011class,mirzasoleiman2013distributed} and variational
inference \citep{djolonga2014map}. Submodular set functions can be
efficiently minimized \citep{iwata2001combinatorial}, and there are
strong guarantees for approximate maximization
\citep{nemhauser1978analysis,krause2012submodular}.

Even though submodularity is most widely considered in the discrete
realm, the notion can be generalized to arbitrary lattices
\citep{fujishige2005submodular}. Recently, \cite{bach2015submodular}
showed how results from {\it submodular set function minimization} can
be lifted to the continuous domain. In this paper, we further pursue
this line of investigation, and demonstrate that results from {\it
  submodular set function maximization} can be generalized as
well. Note that the underlying concepts associated with submodular
function minimization and maximization are quite distinct, and both
require different algorithmic treatment and analysis techniques.

As motivation for our inquiry, 
we firstly give a thorough characterization of the class
of submodular and DR-submodular\footnote{A DR-submodular function is a function with the diminishing returns property, which will be formally defined in \SEC \ref{sec_scf}.} functions in \SEC \ref{sec_scf}. 
We propose the \texttt{weak DR}
property and prove that it is equivalent to  submodularity for general functions. This resolves the question  whether there exists a diminishing-return-style characterization that is equivalent to   submodularity for  all set,   integer-lattice and continuous functions.
We then present two guaranteed
algorithms for maximizing submodular continuous functions in Sections \ref{sec_mono_dr_fun} and \ref{sec_nonmono_2greedy}, respectively.  The first approach, based on the Frank-Wolfe algorithm \citep{frank1956algorithm} and the continuous greedy
algorithm \citep{DBLP:conf/stoc/Vondrak08},
applies to { monotone} DR-submodular functions. It provides a $(1-1/e)$
approximation guarantee under \textit{general} down-closed convex  constraints. We also
provide a second, coordinate-ascent-style algorithm,  which applies to
arbitrary submodular continuous function maximization under box constraints,  and provides a $1/3$ approximation
guarantee. This algorithm is based on the double greedy algorithm \citep{buchbinder2012tight} from
submodular set function maximization. 
In \SEC \ref{sec_app} we  illustrate how submodular
continuous maximization captures various important  applications, ranging from
sensor energy management, 
to influence and revenue maximization, to facility location, 
and non-convex/non-concave quadratic
programming.
Lastly, we
experimentally demonstrate the effectiveness of our algorithms on
several problem instances in \SEC \ref{sec_exp}.

\section{Background and related work}

\paragraph{Notions of submodularity.}
Submodularity is often viewed as a discrete analogue of convexity,
and provides computationally effective structure so that many
discrete problems with this property are efficiently solvable or
approximable.  Of particular interest is a $(1-1/e)$-approximation for
maximizing a monotone submodular set function subject to a cardinality, a
matroid, or a knapsack constraint \citep{nemhauser1978analysis,
  DBLP:conf/stoc/Vondrak08, sviridenko2004note}. 
Another result relevant to this work is unconstrained maximization of  non-monotone submodular set 
functions, for which
\cite{buchbinder2012tight} propose  the   deterministic double greedy algorithm with  1/3 approximation
guarantee, and the randomized double greedy algorithm which achieves  
the tight 1/2 approximation guarantee. 

Although most commonly associated with set functions, in many
practical scenarios, it is natural to consider generalizations of
submodular set functions.
%
\cite{golovin2011adaptive} introduce the notion of adaptive submodularity to
generalize submodular set functions to adaptive policies. 
\cite{kolmogorov2011submodularity} studies tree-submodular functions
and presents a polynomial algorithm for minimizing them. 
For distributive lattices, it is well-known that the combinatorial
polynomial-time algorithms for minimizing a submodular set function
can be adopted to minimize a submodular function over a bounded
integer lattice \citep{fujishige2005submodular}. 
Recently, maximizing
a submodular function over integer lattices has attracted considerable
attention. In particular,  \cite{soma2014optimal} develop  a $(1-1/e)$-approximation algorithm for maximizing a monotone DR-submodular integer-lattice function under a
knapsack constraint. For non-monotone
submodular functions over the bounded integer lattice, \cite{gottschalk2015submodular} provide a 1/3-approximation. Approximation algorithms for
maximizing bisubmodular functions and 
$k$-submodular functions have also
been  proposed by \cite{singh2012bisubmodular,ward2014maximizing}.

\citet{DBLP:journals/mor/Wolsey82} considers maximizing a special class of  submodular continuous functions subject to one knapsack constraint, in
the context of solving  location problems. That class of functions  are additionally required to be monotone, piecewise linear and 
concave. 
\cite{calinescu2007maximizing} and \citet{DBLP:conf/stoc/Vondrak08} discuss 
a subclass of submodular continuous functions, 
which is termed  smooth submodular
functions\footnote{A function $f: [0,1]^n \rightarrow \R$ is smooth submodular if it has second partial derivatives everywhere and all entries of its Hessian matrix are non-positive.}, 
to describe   the multilinear extension of a submodular set function.  They propose  the continuous greedy algorithm, which 
has a $(1-1/e)$ approximation guarantee  on maximizing a smooth
submodular functions under a down-monotone polytope constraint. 
Recently, \cite{bach2015submodular} considers the
minimization of a submodular continuous function, and proves that
efficient techniques from convex optimization may be used for
minimization.  Very recently, \cite{ene2016reduction} provide a   reduction from a
 integer-lattice DR-submodular function maximization problem to a submodular set function maximization problem, which  suggests  a  way to optimize submodular continuous functions over \textit{simple} continuous constriants: Discretize the continuous function and constraint to be an integer-lattice instance, and then optimize it using the reduction. However, 
 for monotone DR-submodular functions maximization,  this method
 can not handle the general continuous constraints discussed in this work, i.e., arbitrary  down-closed convex sets. And for general submodular function maximization,  this method cannot be applied, since the reduction needs the additional  diminishing returns property.  
 Therefore we focus on  continuous methods in this work.  
 
 \paragraph{Non-convex optimization.}
Optimizing  non-convex continuous functions has received renewed
interest in the last decades. 
Recently,
 tensor methods have been used in various non-convex problems, 
e.g., 
learning latent variable models \citep{anandkumar2014tensor} and training neural networks \citep{janzamin2015beating}.
A fundamental problem in non-convex optimization is to reach a stationary point assuming the smoothness of the objective \citep{sra2012scalable,li2015accelerated,reddi2016fast,allen2016variance}.
With extra  assumptions, certain global
convergence results can be obtained. For example, for functions with Lipschitz continuous
Hessians, the regularized Newton scheme of \cite{nesterov2006cubic}
achieves global convergence results for functions 
 with an additional star-convexity property or with an additional gradient-dominance
 property \citep{polyak1963gradient}.  \cite{hazan2015graduated} introduce
 the family of $\sigma$-nice functions and propose a graduated optimization-based algorithm, that
 provably converges to a global optimum for this family of (generally) non-convex functions. 
However, it is typically difficult to verify whether these assumptions hold in real-world problems.  

To the best of our knowledge, this work is the \emph{first} to address the
general problem of monotone and non-monotone submodular maximization over continuous domains. It is also the first 
to propose a sufficient and necessary diminishing-return-style characterization of submodularity for general functions. 
We  propose efficient algorithms with strong approximation guarantees.
We further show that continuous submodularity is a common property of many well-known objectives  and finds various real-world applications.

\paragraph{Notation.}
Throughout this work we assume   $E=\{e_1,e_2,\cdots,
		        e_n\}$ is the ground set of $n$ elements, and
		        $\chara_i\in\R^n$ is the characteristic vector for element $e_i$.  
		        We use boldface letters  $\x\in \R^E$ and $\x\in \R^n$ interchanglebly to indicate  a $n$-dimensional vector, where
		        $x_i$ is the $i$-th entry of $\x$. We use a boldface captial letter $\bmA\in\R^{m\times n}$ to denote a matrix. 
For two vectors $\x,\y\in \R^E$, $\x\leq \y$ means $x_i\leq y_i$ 
for every element $i$ in $E$.
Finally, $\sete{x}{i}{k}$ is the operation of setting the $i$-th element of $\x$
			 to $k$, while keeping all other elements unchanged. 

\section{Characterizations of  submodular continuous functions}
\label{sec_scf}
Submodular continuous functions are defined on   subsets of $\R^n$:
$\X = \prod_{i=1}^n \X_i$, where each $\X_i$ is a compact
subset of $\mathbb{R}$ \citep{topkis1978minimizing,
  bach2015submodular}.
A function $f: \X \rightarrow \R$ is submodular
\textit{iff} for all $(\x, \y)\in \X \times \X$,
\begin{equation}\label{eq1}
  f(\x) + f(\y) \geq f(\x \vee \y) + f(\x \wedge \y),  \quad (\texttt{submodularity})
\end{equation}
where $\wedge$ and $\vee$ are the coordinate-wise minimum and
maximum operations, respectively.  
 Specifically, $\X_i$ could be a finite set, such as $\{0, 1 \}$ (in which case $f(\cdot)$ is called \textit{set} function), or $\{0, \cdots, k_i-1 \}$ (called \textit{integer-lattice} function), 
 where
the notion of continuity is vacuous;  
$\X_i$ can also be  an interval, which is referred to  as a continuous domain. In this work, 
 we consider the interval by default, but it is worth noting that
 the properties introduced in this section can be applied to $\X_i$ being a general compact subset of $\R$. 

When twice-differentiable, $f(\cdot)$
is submodular iff all
off-diagonal entries of its Hessian 
are non-positive\footnote{Notice that an equilavent definition of (\ref{eq1}) is that $\forall \x\in \X$, $\forall i \neq j$ and $a_i, a_j\geq 0$ s.t. $x_i +a_i\in \X_i, x_j+a_j\in \X_j$, it holds $f(\x+a_i\bmchi_i) + f(\x+a_j\bmchi_j) \geq f(\x) + f(\x+a_i\bmchi_i + a_j\bmchi_j)$. With $a_i$ and $a_j$ approaching zero, one  get (\ref{eq2}).} \citep{bach2015submodular},
\begin{equation}\label{eq2}
  \forall \x\in \X, \;\; \frac{\partial^2 f(\x)}{\partial x_i \partial x_j}
  \leq 0, \;\; \forall i \neq j. 
\end{equation}

\begin{wrapfigure}[10]{r}{0.3\textwidth}
	\vspace{-1.1cm}  
	\includegraphics[width=0.3\textwidth]{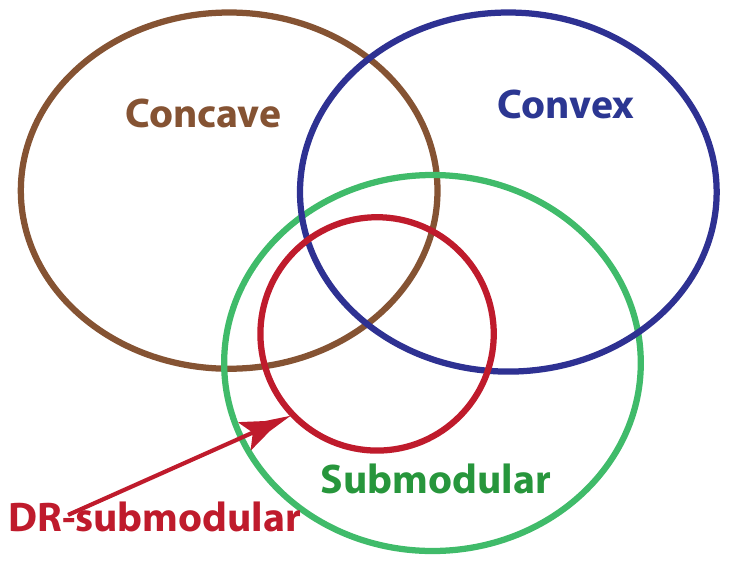}
	\caption{\small Concavity, convexity, submodularity and DR-submodularity.}
	\label{fig_sub}
\end{wrapfigure} 
The class of submodular continuous functions contains a subset of both
convex and concave functions, and shares some useful properties with
them (illustrated in \FIG \ref{fig_sub}).  
Examples include
submodular and convex functions of the form $\phi_{ij}(x_i - x_j)$ for
$\phi_{ij}$ convex; 
submodular and concave functions of the form $\x\!
\mapsto\! g(\sum_{i=1}^{n} \lambda_i x_i)$ for
$g$
concave and $\lambda_i$
non-negative (see \SEC \ref{sec:revenue} for example applications).
Lastly, indefinite quadratic functions of the form $f(\x)
= \frac{1}{2} \x^\trans \bmH \x + \h^\trans \x$ with all off-diagonal entries of
$\bmH$
non-positive are examples of submodular but non-convex/non-concave
functions. Continuous submodularity is preserved under various operations, 
 e.g.,
the sum of two submodular continuous functions is submodular, a submodular
continuous function multiplied by a positive scalar is still submodular. 
Interestingly,  characterizations of  submodular continuous functions are in correspondence to those of convex functions,  which are  summarized in Table \ref{tab_comparison}.
\begin{table}[t]
	\begin{center}
		\begin{tabular}{|l|l|l|l|c|c|}
			\hline
			Properties &    Submodular continuous function $f(\cdot)$ & Convex function $g(\cdot)$, $\forall \lambda\in [0,1]$ \\
			\hline
			\hline 
			$0^{\text{th}}$ order  &  $f(\x) + f(\y) \geq f( \x \vee \y) + f(\x \wedge \y)$ &  $\lambda g(\x) + (1-\lambda)g(\y) \geq g(\lambda \x + (1-\lambda) \y)$   \\
			\hline 
			$1^{\text{st}}$ order & {\texttt{weak DR} (\textcolor{red}{this work}, Definition \ref{def_weakdr})} & $g(\y) - g(\x)\geq  \dtp{\nabla g(\x)}{\y-\x}$   \\
			\hline 
			$2^{\text{nd}}$ order & $\frac{\partial^2 f(\x)}{\partial x_i \partial x_j} \leq 0, \forall i \neq j$  & $\nabla^2 g(\x)  \succeq 0$ (positive semidefinite)   \\
			\hline
		\end{tabular}
		\caption{Comparison of properties of submodular and convex continuous functions}
		\label{tab_comparison}
	\end{center}
\end{table}

In the remainder of this section, we introduce  useful properties of  submodular continuous functions. 
First of all, we generalize the 
\texttt{DR} property (which was introduced
 when  studying  set and   integer-lattice functions) to general functions defined over $\X$. It will soon be clear that the 
\texttt{DR} property defines a subclass of
submodular functions. 

\begin{definition}[\texttt{DR} property and DR-submodular functions]\label{def_dr}
	A function $f(\cdot)$ defined over $\X$ satisfies the  \textit{diminishing returns} (DR)
	property  if  $\forall  \a\leq \b\in \X$, $\forall i\in E$, $\forall k\in \R_+$ s.t. ($k\chara_i+ \a$) and ($k\chara_i + \b$) are still in $\X$, it holds, 
	\begin{equation}\notag 
	f(k\chara_i+ \a) - f(\a) \geq f(k\chara_i + \b) - f(\b).   
	\end{equation}
	This function $f(\cdot)$ is called a  DR-submodular\footnote{Note that DR property implies submodularity and thus the name ``DR-submodular" contains redundant information about submodularity of a function, but we keep this terminology to be consistent with previous literature on submodular integer-lattice functions.} function.
\end{definition}
One immediate observation is that for a differentiable DR-submodular continuous function $f(\cdot)$, we have that $\forall  \a\leq \b\in \X$, $\nabla f(\a)\geq \nabla f(\b)$, i.e., the gradient  $\nabla f(\cdot)$ is an \emph{antitone} mapping from $\R^n$ to $\R^n$. Recently, the \texttt{DR} property is explored by \citet{NIPS2016_6073}  to achieve the 
worst-case competitive ratio for an online  concave maximization problem.  
The \texttt{DR} property is also closely  related to a sufficient condition on 
a concave function $g(\cdot)$ 
\citep[Section 5.2]{bilmes2017deep},  to ensure submodularity of the corresponding 
set function  generated by giving  $g(\cdot)$ boolean input vectors. 


%
It is well known that for set functions, the \texttt{DR} property is equivalent to submodularity, while for 
integer-lattice functions, submodularity does not in general imply the \texttt{DR} property \citep{soma2014optimal,DBLP:conf/nips/SomaY15,soma2015maximizing}. However, it was unclear  whether there exists a diminishing-return-style characterization that is equivalent to  submodularity of  integer-lattice functions. In this work we give a positive answer to this open problem by  proposing the \textit{weak diminishing returns}
(\texttt{weak DR}) property  for general functions defined over $\X$, and prove that \texttt{weak DR} gives a sufficient and necessary condition for a general function to be  submodular.

\begin{definition}[\texttt{weak DR} property]\label{def_weakdr}
	  A  function $f(\cdot)$ defined over $\X$ has the \textit{weak diminishing returns} 
	property (\texttt{weak DR})  if 
	$\forall \a\leq \b\in \X$, $\forall i\in E \text{ s.t. } a_{i} = b_{i}, \forall k\in \R_+$ s.t. $(k\chara_i+\a)$ and $(k\chara_i+\b)$ are still in $\X$, it holds, 
		\begin{equation} \label{def_supp_dr2}
			f(k\chara_i+\a) - f(\a) \geq f(k\chara_i + \b) - f(\b).
		\end{equation}
\end{definition}
%
The following proposition  shows that for all set functions, as well as integer-lattice and continuous  
functions, submodularity is equivalent to   the \NEWDR\
property.
\begin{proposition}[\texttt{submodularity}) $\Leftrightarrow$
  (\NEWDR]\label{lemma_support_dr} A function
  $f(\cdot)$ defined over $\X$
  is submodular iff it satisfies the \textit{weak
    DR} property.
\end{proposition}
All of the proofs can be found in  \SUPP\ \ref{app_proof}.
Given
Proposition \ref{lemma_support_dr}, 
%
 one can treat \NEWDR\ as the first order condition of submodularity: 
 Notice that for a differentiable continuous function $f(\cdot)$ with the 
 \texttt{weak DR} property,   we have that $\forall  \a\leq \b\in \X$, $\forall i\in E \text{ s.t. } a_{i} = b_{i}$, it holds   $\nabla_i f(\a)\geq \nabla_i f(\b)$,  i.e., $\nabla f(\cdot)$ is a weak antitone mapping. 
 %
 %
%
Now we show that the \texttt{DR} property is 
stronger than the \texttt{weak DR} property, and 
the class of DR-submodular functions is a proper subset of that of submodular functions, as indicated  by \FIG \ref{fig_sub}.
\begin{proposition}[\texttt{submodular/weak DR}) + (\texttt{coordinate-wise concave}) 
	$\Leftrightarrow$ (\texttt{DR}]\label{lemma_dr}
A function $f(\cdot)$ defined over $\X$ satisfies the \texttt{DR} property iff
$f(\cdot)$ is submodular  and  coordinate-wise concave,
where the 
\texttt{coordinate-wise concave} property is defined as:  $\forall \x\in \X$, $\forall i\in E$, $\forall k, l\in \R_+$
s.t.  $(k\chara_i+\x), (l\chara_i + \x), ((k+l)\chara_i + \x)$ are still in $\X$, it holds, 
$$f(k\chara_i+\x) - f(\x) \geq f((k+l)\chara_i + \x) - f(l\chara_i + \x),$$
or equivalently (if twice differentiable)
$\frac{\partial^2 f(\x)}{\partial x_i^2} \leq 0, \forall i\in E$.
\end{proposition}
Proposition \ref{lemma_dr} shows that a twice differentiable function $f(\cdot)$
is DR-submodular iff $\forall \x\in \X,  \frac{\partial^2 f(\x)}{\partial x_i \partial x_j}
  \leq 0, \forall i, j\in E$, which does not necessarily imply the concavity of $f(\cdot)$. Given Proposition \ref{lemma_dr},  we also have the  characterizations of
  DR-submodular continuous functions, which are summarized in  Table \ref{tab_dr}. 
  \begin{table}[tbp]
  	\begin{center}
  		\begin{tabular}{|l|l|c|c|c|c|}
  			\hline
  			Properties     &    DR-submodular continuous function $f(\cdot)$, $\forall \x, \y\in \X$\\
  			\hline
  			\hline 
  			$0^{\text{th}}$ order  &  $f(\x) + f(\y) \geq f( \x \vee \y) + f(\x \wedge \y)$, and $f(\cdot)$ is coordinate-wise concave  \\
  			\hline 
  			$1^{\text{st}}$ order &  the {\texttt{DR} property (Definition \ref{def_dr})}  \\
  			\hline 
  			$2^{\text{nd}}$ order & $\frac{\partial^2 f(\x)}{\partial x_i \partial x_j} \leq 0, \forall i , j$ (all entries of the Hessian being non-positive)  \\
  			\hline
  		\end{tabular}
  		\caption{Summarization of properties of DR-submodular continuous functions}
  		\label{tab_dr}
  	\end{center}
  \end{table}

\section{Maximizing monotone \text{DR}-submodular continuous functions}
\label{sec_mono_dr_fun}

In this section,  we present
an algorithm for maximizing a monotone  DR-submodular continuous function subject to a general down-closed convex constraint, i.e.,  $\max_{\x\in \P_{\underline{\bu}}} f(\x)$.
A down-closed convex set $\P_{\underline{\bu}}$ is a convex set $\P$ associated with a
lower bound $\underline{\bu}\in \P$, such that 1) $\forall \y\in \P$, $\underline{\bu} \leq \y$; and 2) $\forall \y\in\P$, $\x\in \R^n$,  $\underline{\bu} \leq \x\leq \y$ implies
$\x\in \P$. 
 Without loss of generality, we assume  $\P$ lies in the postitive orthant and has the lower bound $0$, since otherwise we can  always define a new set $\P' = \{\x \;|\; \x = \y - \underline{\bu}, \y\in \P \}$ in the positive orthant, and a corresponding monotone DR-submdular function $f'(\x) := f(\x + \underline{\bu})$.

Maximizing a monotone DR-submodular function over a down-closed convex constraint has many real-world applications, e.g., influence maximization with continuous 
assignments and sensor energy management. In particular, for influence maximization (see Section \ref{sec_app}), the constraint is a down-closed polytope in the positive orthant: 
$\P = \{\x\in \R^n \;|\; 0\leq \x\leq \bar \bu, \bmA\x \leq \b, \bar \bu \in \R_+^n, \bmA \in \R^{m\times n}_+, \b\in \R^m_+\}$.
We start with the following hardness result:
\begin{proposition}\label{prop_np}
The problem of maximizing a monotone  DR-submodular continuous function subject to 
a general down-closed \emph{polytope} constraint is NP-hard. 
For any $\epsilon >0$, it cannot be approximated in polynomial time within a ratio of $(1-1/e+\epsilon)$ (up to low-order terms), unless RP = NP.
\end{proposition}

Due to the NP-hardness of converging to the global optimum, in the following by ``convergence" we 
mean converging to a point near the global optimum. 
The algorithm is a generalization of    the continuous
greedy algorithm of \cite{DBLP:conf/stoc/Vondrak08} for maximizing a smooth submodular  function, and related to  the convex
Frank-Wolfe algorithm \citep{frank1956algorithm,DBLP:conf/icml/Jaggi13} for minimizing a
convex function. 
\begin{algorithm}[ t]
	\caption{\algname{Frank-Wolfe} variant  for monotone
		{DR}-submodular function
		maximization}\label{alg_sfmax_GradientAscend}
	\KwIn{$\max_{\x\in \P} f(\x)$,
		$\P$ is a down-closed {convex} set in the positive orthant with lower bound $0$,
		prespecified stepsize $\gamma \in (0, 1]$}
	{$\x^0 \leftarrow 0$, $t\leftarrow 0$, $k\leftarrow 0$\tcp*{$k:$ iteration index}}
	\While{$t <  1$}{
		{find $\v^k \text{ s.t. }    \dtp{\v^k}{\nabla f(\x^k)} \geq   \alpha\max_{\v\in\P} \dtp{\v}{ \nabla f(\x^k)} - \frac{1}{2}\delta \cg$\tcp*{$L>0$ is the Lipschitz parameter, $\alpha\in(0, 1]$ is the mulplicative error level, $\delta\in [0, \bar \delta]$ is the additive error level}\label{fw_sub}}
		{find stepsize $\gamma_k\in (0, 1]$, e.g.,  $\gamma_k \leftarrow \gamma $;  
		set $\gamma_k \leftarrow \min\{\gamma_k, 1-t\}$\;}
		{$\x^{k+1}\leftarrow \x^k + \gamma_k \v^k$, $t \leftarrow t + \gamma_k$,  $k\leftarrow k+1$\;\label{step_update}}
	}
	{Return  $\x^K$\tcp*{assuming there are $K$ iterations in total}}
\end{algorithm}
We summarize the \algname{Frank-Wolfe} variant  in \ALG \ref{alg_sfmax_GradientAscend}. 
In each iteration $k$, the algorithm uses the linearization of $f(\cdot)$ as a surrogate, and moves in the direction of the
maximizer of this surrogate function, i.e.
$\v^k=\arg\max_{\v \in \P} \dtp{\v}{ \nabla f(\x^k)}$.
Intuitively, we search for the
direction in which we can maximize the improvement in the function
value and still remain   feasible. 
Finding such a
direction requires maximizing a linear objective at each
iteration.
Meanwhile, it eliminates the need for projecting back to the feasible set
in each iteration, which is an essential step for methods such as
projected gradient ascent. 
The algorithm uses stepsize $\gamma_k$ to update the solution in each iteration, which  can be simply a prespecified value $\gamma$.
%
%
Note that the \algname{Frank-Wolfe} variant can tolerate both multiplicative
error $\alpha$ and additive error $\delta$ when solving the 
subproblem (Step \ref{fw_sub} of \ALG \ref{alg_sfmax_GradientAscend}). Setting $\alpha = 1$ and $\delta = 0$, we recover the error-free case. 

Notice that  the   \algname{Frank-Wolfe} variant  in \ALG \ref{alg_sfmax_GradientAscend} is different from the convex Frank-Wolfe
algorithm mainly in the update direction being used: For  \ALG \ref{alg_sfmax_GradientAscend}, the
update direction (in Step \ref{step_update}) is $\v^k$, while for  convex Frank-Wolfe it is $\v^k - \x^k$, i.e., $\x^{k+1}\leftarrow \x^k + \gamma_k(\v^k - \x^k)$. 
The reason of this difference will soon be clear by exploring the 
property of DR-submodular functions. Specifically, 
\text{DR}-submodular functions are non-convex/non-concave in
general, however, there is certain connection
between DR-submodularity and concavity.
\begin{proposition}\label{prop_concave}
	A {DR}-submodular continuous function $f(\cdot)$ is concave along any
	non-negative direction, and any non-positive direction.
\end{proposition}
Proposition \ref{prop_concave} implies that the univariate
\textit{auxiliary} function
$g_{\x,\v}(\xi):= f(\x+\xi \v), \xi\in \R_+, \v \in \R^E_+$ is concave. {As a result,
	the \algname{Frank-Wolfe} variant can follow a concave
	direction at each step, which is the main reason it uses $\v^k$ as the update direction (notice that $\v^k$ is a non-negative direction).

To derive the approximation guarantee, we need  assumptions on the
“non-linearity” of $f(\cdot)$ over the domain $\P$, which closely
corresponds to a Lipschitz assumption on the derivative of
$g_{\x,\v}(\cdot)$. For a $g_{\x,\v}(\cdot)$ with $L$-Lipschitz continuous
derivative in $[0, 1]$ ($L >0$), we have,
\begin{flalign}\label{eq_cur}
-\frac{L}{2} \xi^2 \leq g_{\x,\v}(\xi) - g_{\x,\v}(0) - \xi \nabla
g_{\x,\v}(0) = f(\x+\xi \v) - f(\x) - \dtp{\xi \v}{\nabla f(\x)}, \forall
\xi \in [0, 1].
\end{flalign}
To prove the approximation guarantee, we first derive the following lemma.
\begin{lemma}\label{lemma_31}
The output solution	$\x^K\in\P$. Assuming $\x^*$ to be the optimal solution, one has,
	\begin{flalign}\label{eq26}
	\dtp{\v^k}{\nabla f(\x^k)}\geq \alpha [f(\x^*) -f(\x^k)] -
	\frac{1}{2}\delta \cg, \ \  \forall k = 0,\cdots, K-1.
	\end{flalign}
\end{lemma}
\begin{theorem}[Approximation guarantee]\label{thm_fw}
	For error levels $\alpha \in (0, 1], \delta\in [0, \bar \delta]$, with $K$ iterations, \ALG
	\ref{alg_sfmax_GradientAscend} outputs $\x^K \in \P$ s.t.,
	\begin{equation}\label{eq8}
		f(\x^K)   \geq  (1-e^{-\alpha})f(\x^*)
		-\frac{\cg}{2} \sum\nolimits_{k=0}^{K-1}\gamma_k^2 -  \frac{\cg\delta}{2} + e^{-\alpha}f(0).	
	\end{equation}
\end{theorem}
Theorem \ref{thm_fw} gives the approximation guarantee for
arbitrary chosen stepsize $\gamma_k$. 
By observing that $\sum_{k=0}^{K-1}\gamma_k =1$ and
$\sum_{k=0}^{K-1}\gamma_k^2 \geq K^{-1}$ (see the proof in \SUPP\ \ref{app_proof_c9}), with constant stepsize,  we obtain the following ``tightest" approximation  bound,
\begin{corollary}\label{cor_9}
For a fixed number of iterations $K$, 
  and  constant stepsize $\gamma_k  =\gamma = K^{-1}$, \ALG
  \ref{alg_sfmax_GradientAscend}  provides the following approximation guarantee: 
  $$f(\x^K)  \geq   (1-e^{-\alpha})f(\x^*)
  -\frac{\cg}{2K}- \frac{\cg\delta}{2} + e^{-\alpha}f(0).$$	
\end{corollary} 
Corollary \ref{cor_9} implies that with a constant stepsize $\gamma$,  1) when $\gamma \rightarrow 0$
($K\rightarrow \infty$), \ALG \ref{alg_sfmax_GradientAscend} will
output the solution with the  worst-case guarantee
$(1-1/e)f(\x^*)$ in the error-free case if $f(0) = 0$; and 2) The \algname{Frank-Wolfe} variant has a sub-linear
convergence rate for monotone {DR}-submodular maximization over a down-closed convex constraint.

\noindent\textbf{Time complexity.}\quad  It can be seen that when using a constant stepsize, \ALG \ref{alg_sfmax_GradientAscend} needs $O(\frac{1}{\epsilon})$ iterations to get $\epsilon$-close to the worst-case guarantee $(1-e^{-1})f(\x^*)$ in the error-free case.
When $\P$ is a polytope
in the positive orthant,  one iteration of \ALG
\ref{alg_sfmax_GradientAscend} costs approximately the same as solving a positive LP, for which a nearly-linear time solver exists \citep{allen2015nearly}.


\vspace{-.2cm}
\section{Maximizing non-monotone submodular continuous functions}
\label{sec_nonmono_2greedy}

The problem of maximizing a general non-monotone submodular continuous function under box constraints\footnote{It is also called ``unconstrained" maximization in the combinatorial optimization community, since the domain $\X$ itself is also a box. Note that the box can be in the negative orthant here.}, i.e., $\max_{\x\in [\underline{\bu}, \bar \bu] \subseteq \X} f(\x)$,  has various real-world applications, including revenue maximization
with continuous assignments, multi-resolution summarization, etc, 
as discussed in Section \ref{sec_app}. 
The following proposition shows the NP-hardness of the problem.

\begin{proposition}\label{prop_np2}
The problem of maximizing a generally non-monotone submodular continuous function subject to 
 box constraints is NP-hard. 
Furthermore, there is no  $(1/2 +\epsilon)$-approximation $\forall \epsilon>0$, unless RP = NP.

\end{proposition}

We now describe our algorithm for maximizing a 
non-monotone  submodular continuous function subject to box constraints. 
It provides a 1/3-approximation, is  inspired by  the double greedy algorithm
of \cite{buchbinder2012tight} and \cite{gottschalk2015submodular}, and can be  viewed as a procedure performing coordinate-ascent on \textit{two} solutions.

We view the process as two particles starting from $\x^0=\underline{\bu}$ and
$\y^0 = \bar{\bu}$, and following a certain ``flow"  toward each
other. 
The pseudo-code is given in \ALG \ref{alg_uscfmax_DoubleGreedy}.
\begin{algorithm}[!t]
  \caption{\algname{DoubleGreedy} algorithm  for maximizing non-monotone submodular continuous functions}\label{alg_uscfmax_DoubleGreedy}
  \KwIn{$\max_{\x\in [\underline{\bu}, \bar \bu]} f(\x)$,
	$f$ is  generally non-monotone,  
    $f(\underline{\bu}) + f(\bar \bu)\geq 0$} {$\x^0 \leftarrow \underline{\bu}$,
    $\y^0 \leftarrow \bar \bu$\;} 
    \For{$k = 1 \rightarrow n$}{ 
    {find
      $\hat u_a$ \text{ s.t. }
      $f(\sete{x^{k-1}}{e_k}{\hat u_a}) \geq \max_{u_a\in[\underline{u}_{e_k}, \bar
        u_{e_k}]} f(\sete{x^{k-1}}{e_k}{ u_a}) - \delta$,
      $\delta_a \leftarrow f (\sete{x^{k-1}}{e_k}{\hat u_a}) -
      f(\x^{k-1}$)\tcp*{$\delta\in
        [0, \bar \delta]$
        is the additive error level }\label{1d_1}} 
{find $\hat u_b$ \text{ s.t. }
      $f(\sete{y^{k-1}}{e_k}{\hat u_b})\geq \max_{u_b\in [\underline{u}_{e_k}, \bar u_{e_k}]} f(\sete{y^{k-1}}{e_k}{u_b}) - \delta$,
      $\delta_b \leftarrow f (\sete{y^{k-1}}{e_k}{\hat u_b}) -
      f(\y^{k-1})$\;\label{1d_2}}
    {\textbf{If} $\delta_a\geq \delta_b$:
      $\x^{k}\leftarrow (\sete{x^{k-1}}{e_k}{\hat u_a})$,
      $\y^{k}\leftarrow (\sete{y^{k-1}}{e_k}{\hat u_a})$ \;}
    {\textbf{Else}: \quad\quad
      $\y^{k}\leftarrow (\sete{y^{k-1}}{e_k}{\hat u_b})$,
      $\x^{k}\leftarrow(\sete{x^{k-1}}{e_k}{\hat u_b})$\;} } {Return
    $\x^n$ (or $\y^n$)\tcp*{note that $\x^n = \y^n$}}
\end{algorithm}
We proceed in $n$ rounds that correspond to some arbitrary order of
the coordinates.
At iteration $k$, we consider solving a one-dimensional (1-D) subproblem
over coordinate $e_k$ for each particle, and moving the particles
based on the calculated local gains toward each other.  Formally, for
a given coordinate $e_k$, we solve a 1-D subproblem to
find the value of the first solution $\x$ along
coordinate $e_k$ that maximizes $f$, i.e.,
$\hat u_a = \arg\max_{u_a} f(\sete{x^{k-1}}{e_k}{u_a}) - f(\x^{k-1})$,
and calculate its marginal gain $\delta_a$.  We then solve
another 1-D subproblem to find the value of the second 
solution  $\y$ along coordinate $e_k$ that maximizes $f$, i.e.,
$\hat u_b = \arg\max_{u_b} f(\sete{\y^{k-1}}{e_k}{u_b}) -
f(\y^{k-1})$,
and calculate the second marginal gain $\delta_b$.  We decide by
comparing the two marginal gains. If changing $x_{e_k}$ to be
$\hat u_a$  has a larger local benefit, we change  \textit{both} $x_{e_k}$ and $y_{e_k}$ to be
$\hat u_a$. Otherwise, we change \textit{both} of them to be
$\hat u_b$.
After $n$ iterations
the particles should meet at point $\x^n = \y^n$, which is the final solution. Note that \ALG \ref{alg_uscfmax_DoubleGreedy} can tolerate  additive error $\delta$ in solving each 1-D subproblem (Steps \ref{1d_1}, \ref{1d_2}).

We would like to emphasize that the  assumptions required by
\algname{DoubleGreedy}  are submodularity of $f$, $f(\underline{\bu}) + f(\bar \bu)\geq 0$ and the (approximate) solvability
of the 1-D subproblem.
For proving the approximation guarantee, the idea is to bound the
loss in the objective value from the assumed optimal objective value between every two consecutive steps,
which is then used to bound the
maximum  loss after $n$ iterations.
\begin{theorem}\label{thm_double}
  Assuming the optimal solution to be $\opt$, the output of \ALG
  \ref{alg_uscfmax_DoubleGreedy} has function value no less than
  $ \frac{1}{3} f(\opt) - \frac{4n}{3}\delta$, where
  $\delta\in [0, \bar \delta]$ is the additive error level for each
  1-D subproblem.
\end{theorem}

\noindent\textbf{Time complexity.}\quad 
It can be seen that the time complexity of  \ALG
\ref{alg_uscfmax_DoubleGreedy} is $O(n*\texttt{cost\_1D})$, where $\texttt{cost\_1D}$ is the cost of solving the 1-D subproblem.  Solving a 1-D subproblem is usually very cheap. For non-convex/non-concave quadratic programming it has a closed form solution.

\section{Examples  of submodular continuous objective
  functions}\label{sec_app}

 In this part, we discuss several concrete problem instances  with their corresponding submodular continuous objective functions.

\noindent\paragraph{Extensions of submodular set functions.}  
The multilinear extension \citep{calinescu2007maximizing} and softmax extension \citep{gillenwater2012near}
are special cases of DR-submodular functions, that  are extensively used
for submodular set function maximization. The Lov{\'a}sz extension \citep{lovasz1983submodular} used for submodular set function minimization 
is both submodular and convex  (see Appendix A in  \cite{bach2015submodular}).

\noindent\paragraph{Non-convex/non-concave quadratic programming (NQP).} NQP problem  of
the form $f(\x) = \frac{1}{2} \x^\trans \bmH \x + \h^\trans \x + c$ under linear constraints
 naturally
arises in many applications, including scheduling \citep{DBLP:journals/jacm/Skutella01}, inventory theory,
and free boundary problems.  A
special class of NQP is the submodular NQP (the minimization of which was studied in \cite{kim2003exact}), in which all
off-diagonal entries of $\bmH$ are required to be non-positive.  In this work, we
mainly use submodular NQPs as synthetic functions for both monotone
DR-submodular maximization and non-monotone submodular maximization.

\noindent\paragraph{Optimal budget allocation with continuous assignments.} 
Optimal budget allocation is  a  special case of the influence maximization
problem.
It  can be modeled as a bipartite graph
$(S,T; W)$, where $S$ and $T$ are collections of advertising channels
and customers, respectively.
The edge weight, $p_{st}\in W$,  represents the influence probability of channel $s$ to customer $t$.
The goal is to distribute the budget (e.g., time for a TV
advertisement, or space of an inline ad) among the source nodes, and
to maximize the expected influence on the potential customers
\citep{soma2014optimal,DBLP:conf/aaai/HatanoFMK15}.  The total influence of customer $t$ from all
channels can be modeled by a proper monotone DR-submodular function
$I_t(\x)$, e.g., $I_t(\x) = 1- \prod_{(s, t)\in W} \left(1-p_{st} \right)^{x_s}$ where $\x\in \R^S_+$ is the budget assignment among the advertising channels.
For a set of $k$ advertisers, let $\x^i\in \R^S_+$ to be the budget assignment for
advertiser $i$, and $\x:= [\x^1,\cdots, \x^k]$ denote the assignments for all the advertisers.    The overall objective is,
\begin{flalign}
 g(\x)= \sum\nolimits_{i=1}^k \alpha_i f(\x^i) ~\text{ with }~ f(\x^i)
 :=\sum\nolimits_{t\in T} I_t(\x^i), \; 0\leq \x^i\leq \bar \bu^i , \forall i = 1,\cdots, k
\end{flalign}
which is monotone DR-submodular.  A concrete application is for  search
marketing advertiser bidding, in which vendors bid for the
right to appear alongside the results of different search
keywords. 
Here, $x^i_s$ is the volume of advertising space allocated to the
advertiser $i$ to show his ad alongside query keyword $s$.  The search
engine company needs to distribute the budget (advertising  space) to all vendors
to maximize their influence on the customers, 
while respecting various constraints. For example, each vendor has a
specified budget limit for advertising, and the ad space associated
with each search keyword can not be too large. 
All such constraints can be
formulated as a down-closed polytope $\P$, hence the \algname{Frank-Wolfe}
variant  can be used to find an approximate solution for the problem $\max_{\x\in \P} g(\x)$.  
Note that one can flexibly add regularizers in designing
$I_t(\x^i)$ as long as it remains monotone DR-submodular. 
For example, adding separable
regularizers of the form $\sum_s \phi(x^i_s)$ does not change the
off-diagonal entries of the Hessian, and hence maintains
submodularity. Alternatively, bounding the second-order
derivative of $\phi(x^i_s)$ ensures DR-submodularity.

\noindent\paragraph{Revenue maximization with continuous assignments.}\label{sec:revenue}
In viral marketing, sellers choose a small subset of buyers to give
them some product for free, to trigger a cascade of further adoptions
through ``word-of-mouth" effects, in order to maximize the total
revenue \citep{hartline2008optimal}.  For some products
(e.g., software), the seller usually gives away the product in the form of
a trial, to be used for free for a limited time period.
In this task, except for deciding whether to choose a user or not, the
sellers also need to decide how much the free assignment
should be, in which the assignments should be modeled as 
continuous variables. 
We call this problem \textit{revenue maximization with continuous
  assignments}.
Assume there are  $q$ products and $n$ buyers/users, 
let $\x^i \in \R_+^n$ to be the assignments of product $i$ to the $n$ users,
 let $\x:= [\x^1,\cdots, \x^q]$ denote the  assignments for the $q$ products. 
The revenue can be modelled as
$g(\x) = \sum_{i=1}^q f(\x^i)$ with
\begin{flalign}\label{eq_re}
  f(\x^i) := \alpha_i \!\!\! \sum\nolimits_{s: x^i_s =0} R_s(\x^i) + \beta_i
  \!\!\!  \sum\nolimits_{t: x^i_t \neq 0} \!\! \phi (x^i_t) \! +\gamma_i \!\!
  \sum\nolimits_{t: x^i_t \neq 0} \!\!\! \bar R_t(\x^i), \quad 
  \; 0\leq \x^i\leq \bar \bu^i, 
\end{flalign}
where $x^i_t$ is the assignment of product $i$ to user $t$ for free, e.g., the amount of free trial time or the amount 
of the product itself. 
$R_s(\x^i)$ models revenue gain from user $s$ who did not receive
the free assignment. It can be some non-negative, non-decreasing
submodular function. $\phi (x^i_t)$ models revenue gain from user $t$
who received the free assignment, since the more one user tries the
product, the more likely he/she will buy it after the trial period.
$\bar R_t(\x^i)$ models the revenue loss from user $t$ (in the free
trial time period the seller cannot get profits), which  can be some
non-positive, non-increasing submodular function. 
With
$\beta \!\!= \!\!\gamma \!\!= \!\!0$, we  recover the classical model of
\cite{hartline2008optimal}.
For products with continuous assignments, usually the cost of the
product does not increase with its amount, e.g., the product as a
software, so we only have the box constraint on each assignment.  The
objective in \EQ \ref{eq_re} is generally
\textit{non-concave/non-convex}, and non-monotone submodular (see \SUPP\ \ref{supp_revenue} for more
details), thus
can be approximately maximized by the proposed \algname{DoubleGreedy} algorithm.
\begin{lemma}\label{revenue}
  If $R_s(\x^i)$ is non-decreasing submodular and $\bar R_t(\x^i)$ is
  non-increasing submodular, then $f(\x^i)$ in \EQ \ref{eq_re} is
  submodular.
\end{lemma}

\noindent\paragraph{Sensor energy management.}
For cost-sensitive outbreak detection in sensor networks \citep{leskovec2007cost}, one
needs to place sensors in a subset of locations selected from all the possible locations $E$, to
quickly detect a set of contamination events $V$, while respecting
the cost constraints of the sensors. 
For each location $e\in E$ and each event $v\in V$,
a value  $t (e, v)$ is provided as the time it takes for the placed sensor in $e$
to detect event $v$. \cite{DBLP:conf/nips/SomaY15} considered the sensors
with discrete energy levels. It is also natural to model the energy levels
of sensors to be a  \emph{continuous} variable $\x\in \R_+^E$. For a sensor 
with energy level $x_e$, the success probability it  detects the event
is $1-(1-p)^{x_e}$,  which models that  by spending one unit of energy 
one has an extra chance of detecting the event with probability $p$. 
In this model, beyond  deciding whether to place a sensor or not, one also
needs to decide the optimal energy levels. Let $t_{\infty} = \max_{e\in E, v\in V}t(e,v)$, let $e_v$ be the first sensor that detects  event $v$ ($e_v$ is a random variable).
One can
define the objective as the expected  detection time that could be \textit{saved},
\begin{flalign}
  f(\x) := \mathbb{E}_{v\in V} \mathbb{E}_{e_v} [t_{\infty} - t(e_v,
  v)],
\end{flalign}
which is a monotone DR-submodular function. 
Maximizing $f(\x)$
w.r.t.  the cost constraints  pursues the goal of finding the
optimal energy levels of the sensors,  to maximize the expected 
 detection time that could be saved.

\noindent\paragraph{Multi-resolution summarization.}
Suppose we have a collection of items, e.g., images
$E = \{ e_1, \cdots, e_n\}$. Our goal is to extract a representative
summary, where representativeness is  defined w.r.t.~a submodular set
function $F:2^E\to \mathbb{R}$. However, instead of returning a single set, our goal
is to obtain summaries at multiple levels of detail or resolution. One
way to achieve this goal is to assign each item $e_i$ a nonnegative score
$x_i$. Given a user-tunable threshold $\tau$, the resulting summary
$S_\tau=\{e_i| x_i \geq \tau\}$ is the set of items with scores
exceeding $\tau$. Thus, instead of solving the discrete problem of
selecting a fixed set $S$, we pursue the goal to optimize over the
scores, e.g., 
to use the following  submodular continuous function,
\begin{flalign} 
  f(\x) =  \sum\nolimits_{i \in E} \sum\nolimits_{j\in E}  \phi(x_j) s_{i,j}
   - \sum\nolimits_{i \in E} \sum\nolimits_{j \in E} x_i x_j s_{i,j},
\end{flalign} 
where $s_{i,j}\geq 0$ is the similarity between items $i,j$, and $\phi(\cdot)$ is a non-decreasing concave function.

\noindent\paragraph{Facility location.}
The classical discrete facility location problem can be naturally generalized to the continuous case where
the scale of a facility is determined
by a continuous value in interval $[0, \bar \bu]$. 
For a set of facilities $E$, let 
$\x\in \R_+^E$ be the scale of all facilities. 
The goal is to decide how large each facility
should be in order to optimally serve 
a set $T$ of  customers. For a facility $s$ of 
scale $x_s$, let $p_{st}(x_s)$ be the value of service
it can provide to customer $t\in T$, where $p_{st}(x_s)$
is a normalized monotone function ($p_{st}(0) = 0$). 
Assuming each customer chooses the facility
with highest value, the total
service provided to all customers is 
$f(\x) = \sum_{t\in T} \max_{s\in E} p_{st}(x_s)$.
It can be shown that $f$ is monotone submodular.

\noindent\paragraph{Other applications.}  Many  discrete submodular problems can be naturally
generalized to the continuous setting with submodular
continuous objectives.  The  maximum
coverage problem  and the problem of text summarization with submodular objectives are among the examples \citep{lin2010multi}. We defer further details to \SUPP\ \ref{supp_more_app}.

\section{Experimental results}
\label{sec_exp}

We compare the performance of our proposed algorithms, the
\algname{Frank-Wolfe} variant and   \algname{DoubleGreedy}, with the following
baselines: a) \algname{Random}: 
uniformly sample $k_s$ solutions from
the constraint set using the hit-and-run
sampler \citep{kroese2013handbook},
and select the best one.
For the constraint set as a very
high-dimensional polytope, this approach is
computationally very expensive. To accelerate sampling from a high-dimensional
polytope, we also use
 b) \algname{RandomCube}:
randomly sample $k_s$ solutions from the {hypercube}, and decrease
their elements until they are inside the polytope. In addition we consider c)
\algname{ProjGrad}: projected gradient ascent with an empirically tuned
step size; and 
d) \algname{SingleGreedy}: for non-monotone submodular functions maximization over a box constraint, we
greedily increase each coordinate, as long as it remains feasible. This
approach is similar to the coordinate ascent method.
In all of the experiments, we use random order of coordinates
for  \algname{DoubleGreedy}. We use constant step size for the \algname{Frank-Wolfe} variant since it gives the tightest approximation guarantee (see Corollary \ref{cor_9}). 
The performance of the methods are evaluated for the following tasks.

\subsection{Results for monotone maximization}

\setkeys{Gin}{width=0.5\textwidth}
\begin{figure}[t]
    \subfloat[\algname{Frank-Wolfe} utility, $K=50$ \label{fig_nqp_iter}]{%
    \hspace{-.3cm}
      \includegraphics[]{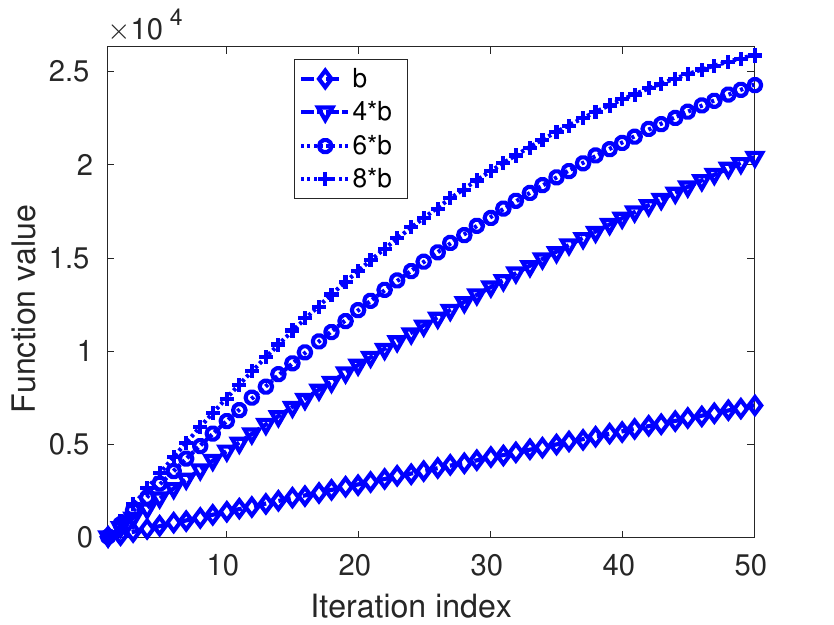}
     }
    \subfloat[Monotone NQP \label{fig_nqp}]{%
    \hspace{-.3cm}
      \includegraphics[]{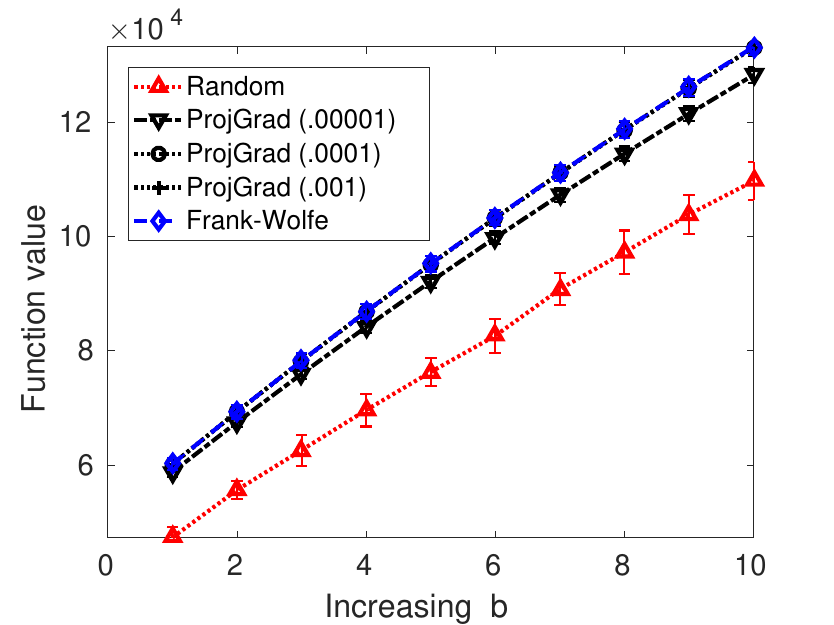}
     }\\
     \subfloat[ Budget allocation \label{fig_inf}]{%
      \hspace{-.3cm}
      \includegraphics[]{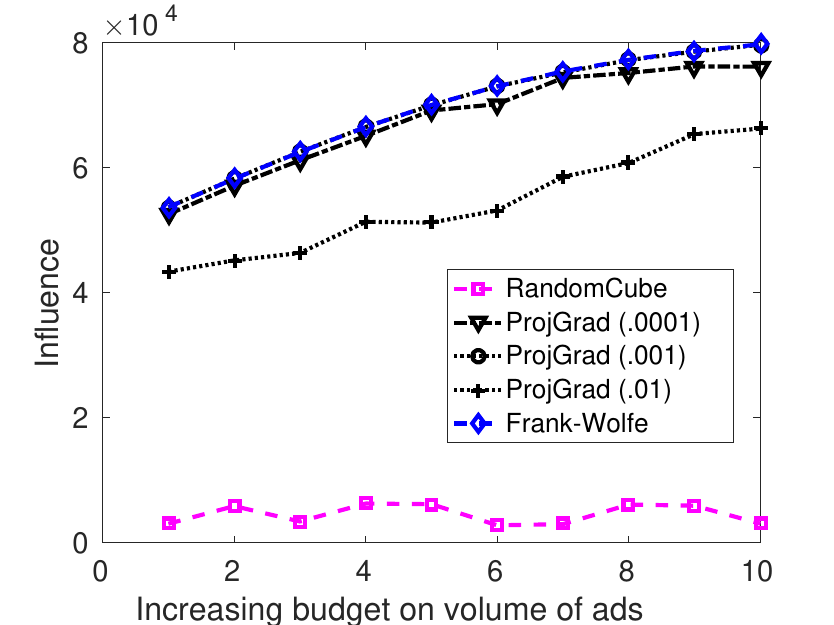}
    }
        \subfloat[Budget allocation  \label{fig_inf2}]{%
         \hspace{-.3cm}
          \includegraphics[]{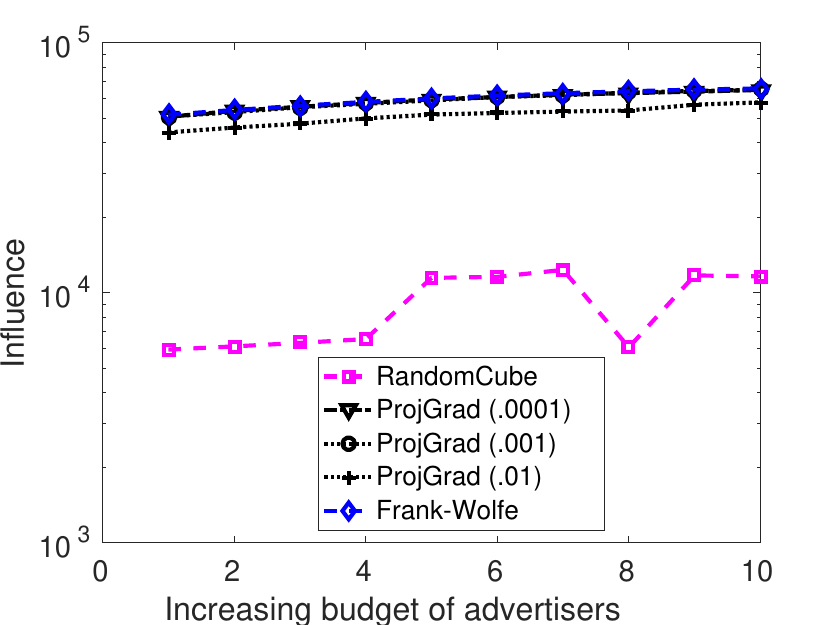}
         }
    \caption{Monotone experiments (both \algname{Frank-Wolfe} and  \algname{ProjGrad} were ran for 50 iterations): a) \algname{Frank-Wolfe} function value  for 4 instances with different $\b$; b) NQP function value returned  w.r.t. different $\b$; c) Influence returned w.r.t.~different budgets on volume of ads;  d) Influence returned w.r.t.~different budgets of advertisers;}
\end{figure}


\textbf{Monotone DR-submodular NQP.}  We randomly generated monotone
DR-submodular NQP functions of the form
$f(\x) = \frac{1}{2} \x^{\trans} \bmH \x + \h^{\trans} \x $, where
$\bmH\in \R^{n \times n}$ is a random matrix with \textit{uniformly} distributed  non-positive
entries in $[-100,0]$,  $n=100$. We further
generated a set of $m = 50$ linear constraints 
to construct the positive polytope
$\P = \{\x\in \R^n | \bmA\x\leq \b, 0\leq \x\leq \bar \bu\}$, where $\bmA$ has uniformly distributed entries in $[0,1]$, $\b= \mathbf{1}, \bar \bu = \mathbf{1}$.
To make the gradient 
non-negative, we set $\h = -\bmH^{\trans} \bar \bu$.
We empirically tuned step size $\alpha_p$ for \algname{ProjGrad} and
ran all algorithms for $50$ iterations. 
 \FIG \ref{fig_nqp_iter} shows the utility
obtained by the \algname{Frank-Wolfe} variant v.s. the 
iteration index for 4 function instances with different values of $\b$.
\FIG \ref{fig_nqp} shows the average utility obtained by different algorithms 
with increasing values of $\b$. The result is the average of 20 repeated experiments. For  \algname{ProjGrad}, we
plotted the curves for three different values of $\alpha_p$. 
One can
observe that the performance of \algname{ProjGrad} fluctuates with
different step sizes. With the best-tuned step size,
\algname{ProjGrad} performs close to the \algname{Frank-Wolfe}
variant.

\paragraph{Optimal budget allocation.}
As our real-world experiments, we used the Yahoo! Search Marketing Advertiser Bidding Data\footnote{\url{https://webscope.sandbox.yahoo.com/catalog.php?datatype=a}}, which consists of 1,000 search keywords, 10,475 customers and 52,567 edges. We considered the 
frequency of (keyword, customer) pairs to estimate the influence probabilities, and
used the average of the bidding prices to put a limit on the budget of each advertiser. 
Since the \algname{Random} sampling was too slow, we compared with the
\algname{RandomCube} method.
  Figures \ref{fig_inf} and \ref{fig_inf2} show the value of the utility function (influence) when varying the budget on volume of ads and on budget of advertisers, respectively.
  Again, we observe that the  \algname{Frank-Wolfe} variant  outperforms the other baselines, and the performance of  \algname{ProjGrad} highly depends on the choice of the step size.

\subsection{Results for non-monotone maximization}

\setkeys{Gin}{width=0.5\textwidth}
\begin{figure}[ht]
	\subfloat[\algname{DoubleGreedy} utility  \label{fig_nnqp1k_ite}]{%
		\hspace{-.3cm}
		\includegraphics[]{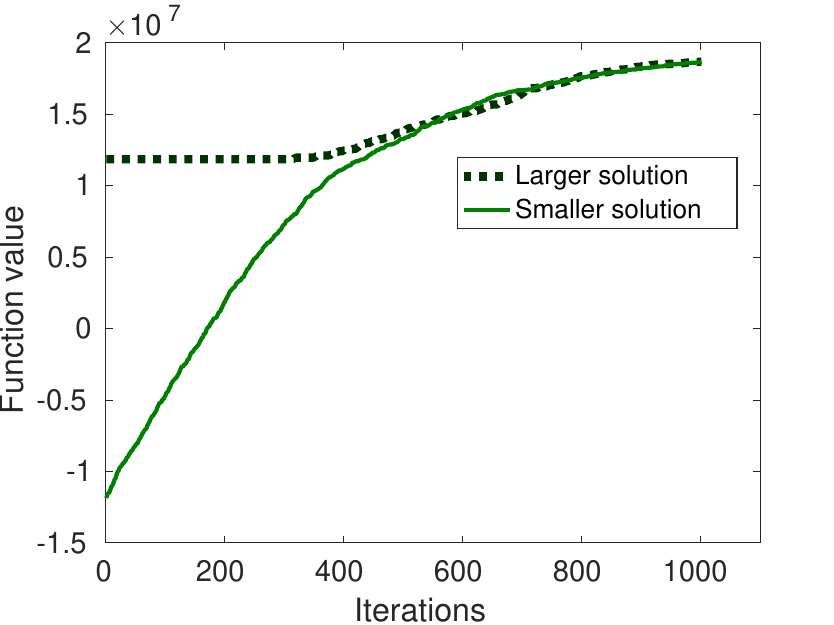}
	}
	\hspace{-.2cm}
	\subfloat[Non-monotone NQP \label{fig_nnqp1k}]{%
		\includegraphics[]{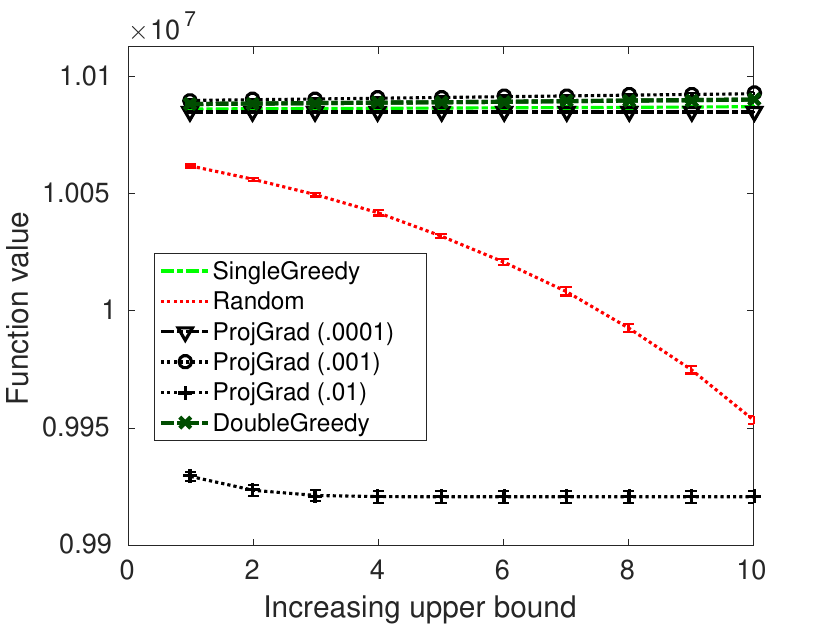}
	}\\
	\subfloat[$\alpha =\beta = \gamma= 10$ \label{fig_revenue2}]{%
		\hspace{-.3cm}
		\includegraphics[]{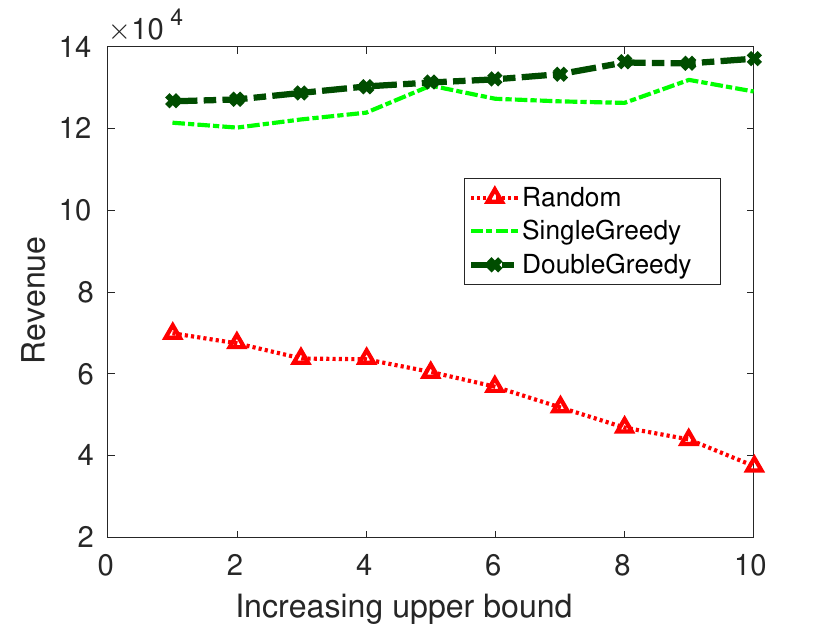}
	}
	\subfloat[$ \alpha =10, \beta  = 5, \gamma = 10$ \label{fig_revenue3}]{%
		\hspace{-.3cm}
		\includegraphics[]{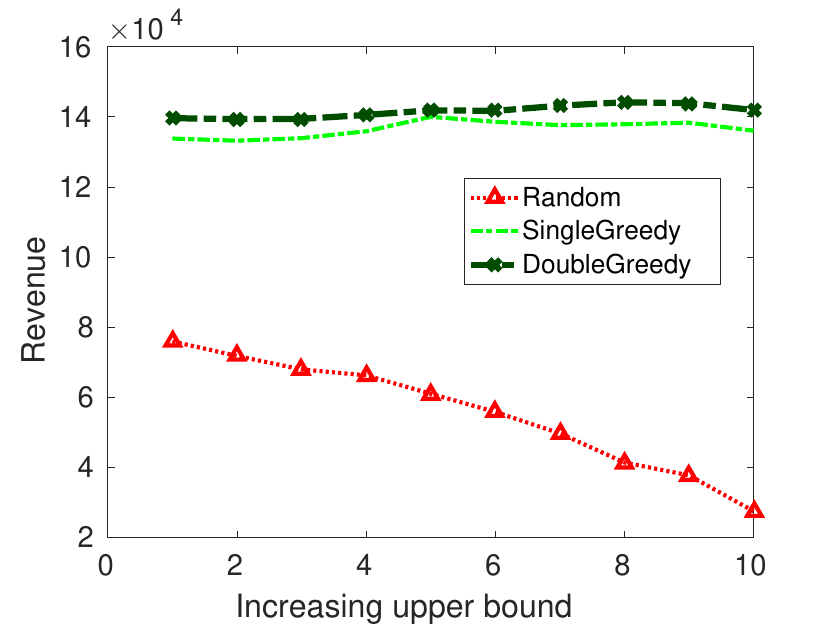}
	}
	\caption{Non-monotone experiments. a) Function
		values of the two intermediate solutions of
		\algname{DoubleGreedy} in each iteration; b)
		Non-monotone NQP function value w.r.t. different
		upper bounds; c, d) Revenue returned with
		different upper bounds $\bar \bu$ on the Youtube social network dataset. }
\end{figure}

\paragraph{Non-monotone submodular NQP.}  We randomly generated non-monotone
submodular NQP functions of the form
$f(\x) = \frac{1}{2} \x^{\trans} \bmH \x + \h^{\trans} \x +c $, where $\bmH\in \R^{n \times n}$
is a sparse matrix with \textit{uniformly} distributed  non-positive off-diagonal entries in $[-10,0]$, $n=1000$. We considered a matrix for which  around
 50\% of the eigenvalues are positive and the other 50\% are negative. We set $\h = -0.2*\bmH^{\trans} \bar \bu$ to make $f(\x)$ non-monotone. 
We then  selected a value for $c$ such that
$f(0) + f(\bar \bu) \geq 0$.  \algname{ProjGrad} was executed for $n$ iterations,  with empirically tuned step sizes.  For the
\algname{Random} method we set $k_s = 1,000$.  \FIG
\ref{fig_nnqp1k_ite} shows the utility of the two intermediate
solutions maintained by \algname{DoubleGreedy}.  One can
observe that they both increase in each iteration.  \FIG
\ref{fig_nnqp1k} shows the values of the utility function for varying  upper
bound $\bar \bu$. The result is the average over 20 repeated experiments.  We can see that \algname{DoubleGreedy} has strong approximation performance, while \algname{ProjGrad}'s results depend on the choice of the step size. With carefully hand-tuned step size,
its performance is comparable to  \algname{DoubleGreedy}.

\paragraph{Revenue maximization.}
Without loss of generality, we considered maximizing the revenue from
selling one product (corresponding to $q=1$, see \SUPP\ \ref{supp_revenue} for more
details on this model).  It can be observed that  the objective in  \EQ \ref{eq_re} is generally non-smooth and \textit{discontinuous}
at any point $\x$ which contains the element of $0$. Since the subdifferential 
can be empty, we cannot use the subgradient-based method and could not
compare with \algname{ProjGrad}.
We performed our experiments on the top 500 largest communities of the
YouTube social
network\footnote{\url{http://snap.stanford.edu/data/com-Youtube.html}}
consisting of 39,841 nodes and 224,235 edges. The edge weights were
assigned according to a uniform distribution $U(0, 1)$.  See \FIG
\ref{fig_revenue2}, \ref{fig_revenue3} for an illustration of revenue
for varying upper bound ($\bar \bu$) and different combinations of the
parameters $(\alpha, \beta, \gamma)$ in the model (\EQ \ref{eq_re}).
For different values of the upper bound,  \algname{DoubleGreedy}
outperforms the other baselines, while \algname{SingleGreedy} maintaining
only one intermediate solution obtained a lower
utility than \algname{DoubleGreedy}.

\section{Conclusion} 
\label{sec_discu}


 

In this paper, we characterized submodular continuous functions, and
proposed two approximation algorithms to efficiently maximize them. In
particular, for maximizing monotone DR-submodular continuous functions subject to general down-closed convex constraints, we proposed
a $(1-1/e)$-approximation algorithm, and for maximizing non-monotone submodular
continuous functions subject to a box constraint, we proposed a 1/3-approximation algorithm. We
demonstrate the effectiveness of our algorithms through a set of
experiments on real-world applications, including budget allocation,
revenue maximization, and non-convex/non-concave quadratic programming, and show that our
proposed methods outperform the baselines in all the
experiments. 
This work demonstrates that submodularity can 
ensure guaranteed optimization in the continuous setting for  problems with (generally) non-convex/non-concave objectives.

\subsubsection*{Acknowledgments}

The authors would like to thank Martin Jaggi for valuable discussions.
This research was partially supported by  ERC StG 307036 and the
Max Planck ETH Center for Learning Systems.

 \clearpage
\bibliographystyle{plainnat}
{
\small
\bibliography{bib_sfm_cont}
}
\newpage
\appendix
%
\begin{center}
\LARGE
\textbf{Appendix}
\end{center}

\section{Proofs of properties of submodular continuous functions}\label{app_proof}

Since $\X_i$ is a compact subset
of $\R$, we denote its lower bound and upper bound to be $\underline{u}_i$ and
$\bar u_i$, respectively in this section. 

\subsection{Alternative formulation of the \NEWDR\ property}

First of all, we will prove that \NEWDR\ has the following  alternative 
	 formulation, which will be used to prove Proposition \ref{lemma_support_dr}.
	 \begin{lemma}[Alternative formulation of \NEWDR]
	 	The  \NEWDR\ property (\EQ \ref{def_supp_dr2}, denoted as \texttt{Formulation I})  has the following equilvalent formulation 
	 	 (\EQ \ref{def_supp_dr}, denoted as  \texttt{Formulation II}): 
	$\forall \a\leq \b\in \X$, $\forall i\in \{i'|a_{i'} = b_{i'}=\underline{u}_{i'} \}, \forall k'\geq l'\geq 0$ s.t. $(k'\chara_i +  \a)$, $(l'\chara_i +  \a)$,  $(k'\chara_i +  \b)$  and $(l'\chara_i +  \b)$ are still in $\X$,  the following inequality is satisfied,
	 		\begin{equation} \label{def_supp_dr}
	 			f(k'\chara_i +  \a) - f(l'\chara_i +  \a) \geq f(k'\chara_i+ \b) - f(l'\chara_i+  \b)     \quad \texttt{(Formulation II)}
	 		\end{equation}
	 \end{lemma}
\begin{proof}

Let $D_1 = \{i| a_i = b_i = \underline{u}_i \}$, $D_2 = \{i|\underline{u}_i < a_i = b_i < \bar u_i \}$, and $D_3 = \{i| a_i = b_i = \bar u_i \}$.

1) \texttt{Formulation II} $\Rightarrow$  \texttt{Formulation I}

When $i\in D_1$, set $l' = 0$ in \texttt{Formulation II} one can
get $f(k'\chara_i+ \a) - f(\a) \geq f(k'\chara_i+ \b) - f(\b)$.

When $i\in D_2$, $\forall k\geq 0$,  let $l' = a_i - \underline{u}_i = b_i- \underline{u}_i >0,$ $k' = k + l' = k +(a_i - \underline{u}_i)$,
and let $\bar \a = (\sete{a}{i}{\underline{u}_i}), \bar \b = (\sete{b}{i}{\underline{u}_i})$. It is easy to see that $\bar \a \leq \bar \b$, and $\bar a_i = \bar{b}_i = \underline{u}_i$. Then from \texttt{Formulation II},
\begin{flalign}\notag 
& f(k'\chara_i + \bar \a) - f(l'\chara_i + \bar \a) = f(k\chara_i +  \a) - f(\a) \\\notag 
& \geq f(k'\chara_i + \bar \b) - f(l'\chara_i + \bar \b) =   f(k\chara_i +  \b) - f(\b).
\end{flalign}
When $i\in D_3$, \EQ \ref{def_supp_dr2} holds trivially.

The above three situations  proves the \texttt{Formulation I}.

2) \texttt{Formulation II} $\Leftarrow$  \texttt{Formulation I}

$\forall \a\leq \b$, $\forall i\in D_1$, one has
$a_i = b_i = \underline{u}_i$. 
$\forall k'\geq l' \geq 0$, let $\hat \a = l'\chara_i + \a, \hat \b = l'\chara_i + \b$, 
let $k = k'-l' \geq 0$, it can be verified that $\hat \a\leq \hat \b$ and $\hat a_i = \hat b_i$, from \texttt{Formulation I},
\begin{flalign}\notag 
&f(k\chara_i + \hat \a) - f(\hat \a) = f(k'\chara_i + \a) - f(l'\chara_i + \a)\\\notag 
\geq & f(k\chara_i + \hat \b) - f(\hat \b) = f(k'\chara_i + \b) - f(l'\chara_i + \b)
\end{flalign}
which proves \texttt{Formulation II}.
\end{proof}	 

\subsection{Proof of Proposition \ref{lemma_support_dr}}

\begin{proof}

	1)  \texttt{submodularity}
	 $\Rightarrow$  \texttt{weak DR}:

Let us prove the \texttt{Formulation II} (\EQ \ref{def_supp_dr}) of \texttt{weak DR}, which is, 

 	$\forall \a\leq \b\in \X$, $\forall i\in \{i'|a_{i'} = b_{i'}=\underline{u}_{i'} \}, \forall k'\geq l'\geq 0$,   the following inequality holds,
	\begin{equation}\notag  
		f(k'\chara_i+ \a) - f(l'\chara_i + \a) \geq f(k'\chara_i+ \b) - f(l' \chara_i + \b).
	\end{equation}

	And 
	$f$ is a submodular function iff $\forall \x, \y\in \X$, $f(\x)+f(\y) \geq f( \x \vee \y) + f(\x \wedge \y)$, so $f(\y) - f(\x\wedge \y) \geq f(\x\vee \y) - f(\x)$. 
	
	Now  $\forall \a \leq \b \in\X$, one can 
	set  $\x = l'\chara_i + \b$ and $\y = k'\chara_i + \a$. It can be easily verified that 
	$\x\wedge \y =l'\chara_i + \a$ and $\x\vee \y = k'\chara_i + \b$. 
	Substituting all the above equalities into
	 $f(\y) - f(\x\wedge \y) \geq f(\x\vee \y) - f(\x)$ one can get 	
	 $f(k'\chara_i+ \a) - f(l'\chara_i + \a) \geq f(k'\chara_i+ \b) - f(l' \chara_i + \b)$.

     2) \texttt{submodularity}
     $\Leftarrow$  \texttt{weak DR}:

Let us use \texttt{Formulation I} (\EQ \ref{def_supp_dr2}) of  \texttt{weak DR} to prove the \texttt{submodularity} property.  

	 $\forall \x, \y\in \X$, 
	let  $D := \{e_1, \cdots, e_d\}$ to be the set of elements for which $y_e > x_e$, let $k_{e_i}: = y_{e_i} - x_{e_i}$. 
	  Now set $\a^0 := \x\wedge \y, \b^0 := \x$ and  $\a^i = (\sete{a^{i-1}}{e_i}{y_{e_i}}) = k_{e_i}\chara_i + \a^{i-1}, \b^i = (\sete{b^{i-1}}{e_i}{y_{e_i}}) = k_{e_i}\chara_i + \b^{i-1}$, for $i = 1, \cdots,  d$. 
	  One can verify that 
	   $\a^i\leq \b^i, a^i_{e_{i'}} = b^i_{e_{i'}}$ for all $i'\in
	 	 D, i=0, \cdots, d$, and
	 	 that $\a^d = \y, \b^d = \x\vee \y$.

	  Applying   \EQ \ref{def_supp_dr2} of the \texttt{weak DR} property for  $i = 1,\cdots, d$ one can get
	 \begin{flalign}\notag 
	 &f(k_{e_1}\chara_{e_1} +  \a^0) - f(\a^0) \geq  f(k_{e_1}\chara_{e_1} + \b^0) - f(\b^0) \\\notag 
	 &f(k_{e_2}\chara_{e_2} + \a^1) - f(\a^1) \geq  f(k_{e_2}\chara_{e_2} + \b^1) - f(\b^1) \\\notag 
	 &\cdots\\\notag
	 &f(k_{e_d}\chara_{e_d} + \a^{d-1}) - f(\a^{d-1}) \geq  f(k_{e_d}\chara_{e_d} + \b^{d-1}) - f(\b^{d-1}). 
	 \end{flalign}
	 Taking a sum over all the above $d$ inequalities, one can get
	 \begin{flalign}\notag 
	 & f(k_{e_d}\chara_{e_d} + \a^{d-1}) - f(\a^{0}) \geq  f(k_{e_d}\chara_{e_d} + \b^{d-1}) - f(\b^{0}) \Leftrightarrow\\\notag 
	 & f(\y) - f(\x\wedge \y) \geq f(\x\vee \y) - f(\x) \Leftrightarrow\\\notag 
	 & f(\x) + f(\y) \geq f(\x\vee \y) + f(\x\wedge \y),
	 \end{flalign}
	 which proves the submodularity. 	 
%
%
%
\end{proof}

\subsection{Proof of Proposition \ref{lemma_dr}}

\begin{proof}

	1) \texttt{submodular} + \texttt{coordinate-wise concave}
		$\Rightarrow$ \texttt{DR}:
	
	From coordinate-wise concavity we have $f(\a + k\chara_i) - f(\a) \geq f(\a+(b_i - a_i + k)\chara_i) - f(\a+(b_i - a_i)\chara_i)$. Therefore, to prove \text{DR} it suffices to show that 
	\begin{flalign}\label{eq_12}
		f(\a+(b_i - a_i + k)\chara_i) - f(\a+(b_i - a_i)\chara_i) \geq f(\b + k\chara_i) - f(\b).
	\end{flalign}
	Let $\x:=\b, \y:=(\a+(b_i - a_i + k)\chara_i)$, so $\x\wedge\y  = (\a+(b_i - a_i)\chara_i), \x\vee \y = (\b + k\chara_i)$.
	From submodularity, one can see that inequality \ref{eq_12} holds.
	
	2) \texttt{submodular} + \texttt{coordinate-wise concave}
			$\Leftarrow$ \texttt{DR}:

	  From \texttt{DR} property, the  \texttt{weak DR} (\EQ \ref{def_supp_dr2}) property is implied, which 
	equivalently proves the \textit{submodularity} property.
	
	To prove \textit{coordinate-wise concavity}, one just need to set $\b:=\a+l\chara_i$, then it reads $f(\a + k\chara_i) - f(\a) \geq f(\a + (k+l)\chara_i) - f(\a + l\chara_i)$. 
\end{proof}

\section{Proofs for the monotone DR-submodular continuous functions maximization}

\subsection{Proof of Proposition \ref{prop_np}}

\begin{proof}

On a high level, the  proof idea follows from the  reduction from the problem of
maximizing a monotone submodular set function subject to cardinality constraints. 

Let us denote $\Pi_1$ as the problem of  maximizing a monotone submodular set function subject to cardinality constraints, and $\Pi_2$ as 
the problem  of maximizing a monotone  DR-submodular continuous function under
general down-closed polytope constraints.
Following \cite{DBLP:journals/siamcomp/CalinescuCPV11}, there exist an algorithm $\A$ for $\Pi_1$
that consists of a polynomial time computation in addition to
polynomial number of subroutine calls to an algorithm for $\Pi_2$. For details
on $\A$ see the following.

First of all, the multilinear extension \citep{calinescu2007maximizing}
of a monotone submodular set function is a monotone 
submodular continuous function, and it is coordinate-wise
linear, thus falls into a special case of monotone DR-submodular continuous functions. 

So the algorithm $\A$ can be: 1) Maximize the multilinear 
extension of the submodular set function over the  
matroid polytope associated with the cardinality constraint, which can be achieved 
 by solving an instance of  $\Pi_2$. 
We call the solution obtained the fractional solution; 2)  Round the fractional 
solution to a feasible integeral solution using polynomial 
time rounding technique in \citet{ageev2004pipage,calinescu2007maximizing} (called  the pipage 
 rounding). Thus we prove the 
 reduction from $\Pi_1$ to $\Pi_2$.

Our reduction algorithm  $\A$ implies the NP­-hardness and
inapproximability of problem $\Pi_2$.

For the NP-hardness, because $\Pi_1$  is well­ known to be NP­-hard
 \citep{calinescu2007maximizing,feige1998threshold}, so $\Pi_2$ is NP-hard as well.

For the inapproximability: Assume there exists a polynomial  algorithm ${\mathscr B}$ that can solve $\Pi_2$
better than $1-1/e$, then we can use ${\mathscr B}$ as the subroutine algorithm in the
reduction, which implies that one can solve $\Pi_1$ better than $1-­
1/e$. Now we slightly adapt the proof of inapproximability on
max­-k­-cover in \citet{feige1998threshold}, since max­-k-­cover is a special case of $\Pi_1$.
According to Theorem 5.3 in \citet{feige1998threshold} and our reduction $\A$, we have a
reduction from approximating 3SAT­-5 to problem $\Pi_2$. Using the
rest proof of Theorem 5.3 in \citet{feige1998threshold}, we reach the result that one cannot
solve $\Pi_2$ better than  $1-1/e$,  unless RP = NP.
%
\end{proof}

\subsection{Proof of Proposition \ref{prop_concave}}\label{supp_prop_concave}

\begin{proof}
Consider a function $g(\xi):= f(\x+\xi \v^*), \xi\geq 0, \v^* \geq 0$. $\frac{d g(\xi)}{d \xi} = \dtp{\v^*}{\nabla f(\x+\xi \v^*)}$.

$g(\xi)$ is concave $\Leftrightarrow$ 
\begin{flalign}\notag 
	\frac{d^2 g(\xi)}{d \xi^2} = (\v^*)^\trans \nabla^2 f(\x+\xi \v^*) \v^* = \sum_{i\neq j} v^*_i v^*_j \nabla^2_{ij} f + \sum_i (v_i^*)^2\nabla_{ii}^2f \leq 0
\end{flalign}
The non-positiveness of $\nabla^2_{ij} f $ is ensured by submodularity of $f(\cdot)$, and the non-positiveness of $\nabla^2_{ii} f $ results from the coordinate-wise concavity of $f(\cdot)$.

The proof of concavity along any non-positive direction is similar, which is omitted here. 
\end{proof}

\subsection{Proof of Lemma \ref{lemma_31}}
\begin{proof}
	It is easy to see that 
	$\x^K$ is a convex linear combination of points in $\P$, so $\x^K\in\P$.
	
	Consider the point $\v^*:=(\x^*\vee \x) - \x = (\x^* - \x)\vee 0\geq 0$. Because $\v^*\leq \x^*$ and $\P$ is down-closed, we get $\v^*\in \P$.
	By monotonicity, $f(\x+\v^*) = f(\x^*\vee \x) \geq f(\x^*)$.
	
	Consider the  function $g(\xi):= f(\x+\xi \v^*), \xi\geq 0$. $\frac{d g(\xi)}{d \xi} =     
	\dtp{\v^*}{\nabla f(\x+\xi \v^*)}$.
	From Proposition \ref{prop_concave}, 
	$g(\xi)$ is concave, hence
	\begin{flalign}\notag 
	g(1) - g(0) = f(\x+\v^*) - f(\x) \leq \frac{d g(\xi)}{d \xi} \Bigr|_{\xi = 0} \times 1 =  \dtp{\v^*}{ \nabla f(\x)}.
	\end{flalign}
	Then one can get
	\begin{flalign}\notag 
	&\dtp{\v}{\nabla f(\x)} \overset{(a)}{\geq} \alpha \dtp{\v^*}{ \nabla f(\x)} -\frac{1}{2}\delta \cg \geq \\\notag 
	&\alpha (f(\x+\v^*) - f(\x)) -\frac{1}{2}\delta \cg \geq \alpha (f(\x^*) -f(\x)) -\frac{1}{2}\delta \cg
	\end{flalign}
	where $(a)$ is from the selection rule  in Step 3 of  \ALG \ref{alg_sfmax_GradientAscend}.
\end{proof}

\subsection{Proof of Theorem \ref{thm_fw}}
\begin{proof}
	From the Lipschitz continuous derivative assumption of $g(\cdot)$ (\EQ \ref{eq_cur}): 
	\begin{flalign}\notag 
	f(\x^{k+1}) - f(\x^k) & = f(\x^k + \gamma_k  \v^k) - f(\x^k) \\\notag 
	&=	g(\gamma_k) - g(0) \\ \notag
	&\geq  \gamma_k   \dtp{\v^k}{\nabla f(\x^k)} - \frac{\cg}{2}\gamma_k^2 \quad (\text{Lipschitz assumption in \EQ \ref{eq_cur}})  \\\notag
	&\geq \gamma_k \alpha [f(\x^*) - f(\x^k)] - \frac{1}{2}\gamma_k\delta \cg - \frac{\cg}{2}\gamma_k^2   \quad (\text{Lemma \ref{lemma_31}})
	\end{flalign}
	After rearrangement, 
	\begin{flalign}\notag
	f(\x^{k+1}) - f(\x^*) \geq (1-\alpha\gamma_k) [f(\x^k) - f(\x^*)]- \frac{1}{2}\gamma_k\delta \cg - \frac{\cg}{2}\gamma_k^2 
	\end{flalign}
	Therefore,
	\begin{flalign}\notag
	f(\x^K) - f(\x^*) \geq  \prod_{k=0}^{K-1} (1-\alpha\gamma_k)[f(0) - f(\x^*)] - \frac{\delta L}{2} \sum_{k=0}^{K-1}\gamma_k - \frac{L}{2} \sum_{k=0}^{K-1}\gamma_k^2   .
	\end{flalign}
One can observe that  $\sum_{k=0}^{K-1}\gamma_k = 1$, and 
since $1-y \leq e^{-y}$ when $y\geq 0$, 
		\begin{flalign}\notag
		 f(\x^*) - f(\x^K)  &\leq   [f(\x^*) - f(0)]e^{-\alpha \sum_{k=0}^{K-1}\gamma_k} + \frac{\delta L}{2} + \frac{L}{2} \sum_{k=0}^{K-1}\gamma_k^2 \\\notag
		 & = [f(\x^*) - f(0)]e^{-\alpha} + \frac{\delta L}{2} + \frac{L}{2} \sum_{k=0}^{K-1}\gamma_k^2.   
		\end{flalign}
	After rearrangement, we get,
	$$f(\x^K)  \geq  (1-1/e^{\alpha})f(\x^*)  
	-\frac{\cg}{2} \sum_{k=0}^{K-1}\gamma_k^2 -  \frac{\cg\delta}{2} + e^{-\alpha}f(0).$$
\end{proof}

\subsection{Proof of Corollary \ref{cor_9}}\label{app_proof_c9}

\begin{proof}
	Fixing $K$, to reach the tightest bound in \EQ \ref{eq8} amounts to solving the following problem:
	\begin{flalign}\notag 
		&\min \sum_{k=0}^{K-1}\gamma_k^2\\\notag
		&\text{ s.t. } \sum_{k=0}^{K-1}\gamma_k = 1, \gamma_k \geq 0.
	\end{flalign}
Using Lagrangian method, let $\lambda$ be the Lagrangian multiplier, then $$L(\gamma_0,
\cdots, \gamma_{K-1}, \lambda) = \sum_{k=0}^{K-1}\gamma_k^2 + \lambda \left[\sum_{k=0}^{K-1}\gamma_k - 1\right].$$
It can be easily verified that when $\gamma_0 = \cdots =\gamma_{K-1} = K^{-1}$, 
$\sum_{k=0}^{K-1}\gamma_k^2$ reaches the minimum (which is $K^{-1}$). Therefore we obtain the tightest worst-case bound in Corollary \ref{cor_9}.
\end{proof}

\section{Proofs 
for the non-monotone submodular continuous functions maximization
	}

\subsection{Proof of Proposition \ref{prop_np2}}

\begin{proof}

The main proof  follows from the  reduction from the problem of
maximizing an unconstrained non-monotone submodular set function.

Let us denote $\Pi_1$ as the problem of  maximizing an unconstrained non-monotone submodular set function, and $\Pi_2$ as 
the problem  of maximizing a box constrained non-monotone   submodular continuous function.
Following the Appendix A of  \cite{buchbinder2012tight}, there exist an  algorithm $\A$ for $\Pi_1$
that consists of a polynomial time computation in addition to
polynomial number of subroutine calls to an algorithm for $\Pi_2$. For details
see the following.

Given a submodular set  function $F: 2^{\groundset}\rightarrow \R_+$, its
 multilinear extension \citep{calinescu2007maximizing}
 is a function $f: [0,1]^\groundset \rightarrow \R_+$, whose value 
 at a point $\x\in [0,1]^\groundset$ is the expected value of $F$ over 
 a random subset $R(\x)\subseteq \groundset$, where $R(\x)$ contains
 each element $e\in \groundset$ independently with probability $x_e$.
 Formally, $f(\x):= \mathbb{E} [R(\x)] = \sum_{S\subseteq \groundset} F(S) \prod_{e\in S} x_e\prod_{e'\notin S}(1-x_{e'})$.
 It can be easily seen that $f(\x)$ is a non-monotone 
submodular continuous function.

Then the algorithm $\A$ can be: 1) Maximize the multilinear extension
$f(\x)$ over the box constraint $[0, 1]^\groundset$, which can be achieved by
solving an instance of $\Pi_2$. Obtain the fractional solution $\hat \x\in [0, 1]^n$; 2) Return the random set $R(\hat \x)$. According to the definition
of multilinear extension, the expected value of $F(R(\hat \x))$ is
$f(\hat \x)$.  Thus proving the reduction from $\Pi_1$ to $\Pi_2$.

Given the reduction, the hardness result follows from the hardness
of unconstrained non-monotone submodular set function maximization. 

The inapproximability result comes from that of the unconstrained non-monotone submodular set function maximization in \cite{feige2011maximizing}  and \cite{dobzinski2012query}.
\end{proof}	

\subsection{Proof of Theorem  \ref{thm_double}}\label{supp_double}

To better illustrate the proof, we reformulate  \ALG \ref{alg_uscfmax_DoubleGreedy} into its  equivalent form in  \ALG \ref{alg_uscfmax_DoubleGreedy_a}, where we split the update into two steps: when $\delta_a\geq \delta_b$, update $\x$ first  
   while keeping $\y$ fixed， and then update $\y$ first while
   keeping $\x$ fixed ($\x^{i}\leftarrow (\sete{x^{i-1}}{e_i}{\hat u_a})$, $\y^{i}\leftarrow \y^{i-1}$;  $\x^{i+1}\leftarrow \x^{i}$, $\y^{i+1}\leftarrow (\sete{y^{i}}{e_i}{\hat u_a})$ ), when $\delta_a < \delta_b$,
   update $\y$ first.  This iteration index change
   is only used to ease the analysis.

   To prove the theorem, we first prove the following Lemmas.

\begin{algorithm}[h]
	\caption{\algname{DoubleGreedy} algorithm  reformulation (for analysis only)}\label{alg_uscfmax_DoubleGreedy_a}
	\KwIn{$\max f(\x)$,  $\x\in [\underline{\bu}, \bar \bu]$,   $f$ is  generally non-monotone,  $f(\underline{\bu}) + f(\bar \bu)\geq 0$}
	{$\x^0 \leftarrow \underline{\bu}$, $\y^0 \leftarrow \bar \bu$\;}
	\For{$i = 1, 3, 5,\cdots, 2n-1$}{
		{find $\hat u_a$ \text{ s.t. }  $f(\sete{x^{i-1}}{e_i}{\hat u_a}) \geq \max_{u_a\in[\underline{u}_{e_i}, \bar u_{e_i}]}  f(\sete{x^{i-1}}{e_i}{u_a}) - \delta$, 
		$\delta_a  \leftarrow  f(\sete{x^{i-1}}{e_i}{\hat u_a}) - f(\x^{i-1})$ \tcp*{$\delta\in [0, \bar \delta]$ is the additive error level. 
		}}
			{find $\hat u_b$ $\text{ s.t. } f(\sete{y^{i-1}}{e_i}{\hat u_b})\geq  \max_{u_b\in[\underline{u}_{e_i}, \bar u_{e_i}]} f(\sete{y^{i-1}}{e_i}{u_b}) - \delta$, 
			$\delta_b  \leftarrow  f(\sete{y^{i-1}}{e_i}{\hat u_b}) - f(\y^{i-1})$ 
			\;}
		\If{$\delta_a\geq \delta_b$}{
			{$\x^{i}\leftarrow (\sete{x^{i-1}}{e_i}{\hat u_a})$, $\y^{i}\leftarrow \y^{i-1}$  \;}
					{$\x^{i+1}\leftarrow \x^{i}$, $\y^{i+1}\leftarrow (\sete{y^{i}}{e_i}{\hat u_a})$  \;}}
			\Else{
		{$\y^{i}\leftarrow (\sete{y^{i-1}}{e_i}{\hat u_b})$, $\x^{i}\leftarrow \x^{i-1}$\;}
				{$\y^{i+1}\leftarrow \y^{i}$, $\x^{i+1}\leftarrow (\sete{x^{i}}{e_i}{\hat u_b})$\;}}
	}
	{Return  $\x^{2n}$ (or $\y^{2n})$ \tcp*{note that  $\x^{2n} = \y^{2n}$}}
\end{algorithm}

Lemma \ref{lemma_42} is used to demonstrate that
the objective value of each intermediate solution
is non-decreasing,
\begin{lemma}\label{lemma_42}
$\forall  i = 1, 2,\cdots, 2n$, one has,
	\begin{flalign}
		f(\x^i) \geq f(\x^{i-1}) -\delta, \;\; f(\y^i) \geq f(\y^{i-1})-\delta.
	\end{flalign}
\end{lemma}

\begin{proof}[Proof of Lemma \ref{lemma_42}]
	Let $j:= e_i$ be the coordinate that is going to be changed. From submodularity,
	\begin{flalign}\notag
		f(\sete{x^{i-1}}{j}{\bar u_j}) + f(\sete{y^{i-1}}{j}{\underline{u}_j}) \geq f(\x^{i-1}) + f(\y^{i-1})
	\end{flalign}
	So one can verify that $\delta_a + \delta_b \geq -2\delta$. Let us consider the 
	following two situations: 
	
	\textcolor{emp_color}{1)} If $\delta_a \geq  \delta_b$, $\x$ is changed first.
	
	We can see that the Lemma holds for the first change (where $\x^{i-1}\rightarrow \x^i, \y^i = \y^{i-1}$).  For the second change,  we are left to prove $f(\y^{i+1}) \geq f(\y^i) -\delta$.
	From submodularity:
	\begin{flalign}
	f(\sete{y^{i-1}}{j}{\hat u_a}) + f(\sete{x^{i-1}}{j}{\bar u_j}) \geq f(\sete{x^{i-1}}{j}{ \hat u_a }) + f(\y^{i-1})
	\end{flalign}
	Therefore, $f(\y^{i+1}) - f(\y^i) \geq f(\sete{x^{i-1}}{j}{\hat u_a}) -  f(\sete{x^{i-1}}{j}{ \bar u_j}) \geq -\delta$, the last inequality comes from the selection rule of $\delta_a$ in the algorithm.
		
		\textcolor{emp_color}{2)} Otherwise,  $\delta_a <  \delta_b$, $\y$ is changed first.

The Lemma holds for the first change ($\y^{i-1}\rightarrow \y^i, \x^i = \x^{i-1}$). 
For the second change,  we are left to prove $f(\x^{i+1}) \geq f(\x^i) - \delta$. From
submodularity,
\begin{flalign}
		f(\sete{x^{i-1}}{j}{\hat u_b}) + f(\sete{y^{i-1}}{j}{\underline{u}_j}) \geq f(\sete{y^{i-1}}{j}{\hat  u_b}) + f(\x^{i-1}), 
\end{flalign}
So $f(\x^{i+1}) - f(\x^i) \geq f(\sete{y^{i-1}}{j}{\hat u_b}) -  f(\sete{y^{i-1}}{j}{\underline{u}_j}) \geq -\delta$, the last inequality also comes from the selection rule of $\delta_b$.
\end{proof}

Let 
$OPT^i := (\opt\vee \x^i)\wedge \y^i$, it is easy to observe that $OPT^0 = \opt$ and $OPT^{2n} = \x^{2n} = \y^{2n}$. 
\begin{lemma}\label{lemma_43}
$\forall  i = 1, 2,\cdots, 2n$, it holds, 
	\begin{flalign}
		f(OPT^{i-1}) - f(OPT^{i}) \leq f(\x^i) - f(\x^{i-1}) + f(\y^i) - f(\y^{i-1}) +2\delta.
	\end{flalign}
\end{lemma}
Before proving Lemma \ref{lemma_43}, 
let us get some intuition about it. 
We can see that
when changing $i$ from 0 to $2n$, the objective value
changes from the optimal value $f(\opt)$ to the value returned by the algorithm: $f(\x^{2n})$. Lemma \ref{lemma_43} is then used to bound the objective loss
from the assumed optimal objective in each iteration. 

\begin{proof} 

	Let $j:= e_i$ be the coordinate that will be changed. 
	
	 First of all, let us assume $\x$ is changed, $\y$ is kept unchanged ($\x^i \neq \x^{i-1}, \y^i = \y^{i-1}$), this could happen in four situations: \textcolor{emp_color}{1.1)} $x^i_j \leq \opti{j}$ and  $\delta_a\geq \delta_b$; \textcolor{emp_color}{1.2)} $x^i_j \leq \opti{j}$ and $\delta_a < \delta_b$; \textcolor{emp_color}{2.1)} $x^i_j >  \opti{j}$ and $\delta_a\geq \delta_b$;  \textcolor{emp_color}{2.2)} $x^i_j > \opti{j}$ and $\delta_a < \delta_b$.
	 Let us
	 	prove the four situations one by one. 
	
	\textbf{If} $x^i_j \leq \opti{j}$,  the Lemma holds in the following two situations: 
	
 {\textcolor{emp_color}{\bf 1.1)}}	 When $\delta_a \geq \delta_b$, it happens in
	the first change: $x^i_j =  \hat u_a \leq \opti{j}$, so $OPT^i = OPT^{i-1}$; According to
		Lemma \ref{lemma_42}, $\delta_a + \delta_b \geq -2\delta$, so 
	$f(\x^i) - f(\x^{i-1}) + f(\y^i) - f(\y^{i-1}) +2\delta \geq 0$, so the Lemma holds;

	 \textcolor{emp_color}{\bf 1.2)} When $\delta_a < \delta_b$, it happens in
	the second change: $x^i_j = \hat u_b \leq \opti{j}, y^i_j = y^{i-1}_j = \hat u_b$, and since $OPT^{i-1} = (\opt\vee \x^{i-1})\wedge \y^{i-1}$,  so $OPT^{i-1}_j =\hat u_b$ and $OPT^i_j =\hat u_b$,
	so one still has $OPT^i = OPT^{i-1}$. So it amouts to prove
	that $\delta_a + \delta_b \geq -2\delta$, which is true according to
	Lemma \ref{lemma_42}.
	
	\textbf{Else if} $x^i_j > \opti{j}$, it holds that $OPT^i_j = x^i_j$, all other coordinates of $OPT^{i-1}$
	remain unchanged.  The Lemma holds in the following two situations: 
	
	\textcolor{emp_color}{\bf 2.1)} When $\delta_a \geq \delta_b$, it happens in the 
	first change. One has $OPT^i_j = x^i_j = \hat u_a$,
	  $x^{i-1}_j = \underline{u}_j$, so $OPT^{i-1}_j = \opti{j}$. And
	  $x^i_j =\hat u_a > \opti{j}, y^{i-1}_j = \bar u_j$. 
	From submodularity,
	\begin{flalign}\label{eq_a}
	f(OPT^i) + f(\sete{y^{i-1}}{j}{\opti{j}}) \geq f(OPT^{i-1}) + f(\sete{y^{i-1}}{j}{\hat u_a})
	\end{flalign}
	Suppose by virtue of contradiction that,
	\begin{flalign}\label{eq_b}
	f(OPT^{i-1}) - f(OPT^{i}) > f(\x^i) - f(\x^{i-1})  + 2\delta 
	\end{flalign}
	Summing \EQ \ref{eq_a} and \ref{eq_b} we get:
	\begin{flalign}\label{eq_19}
	0 > f(\x^i) - f(\x^{i-1})  + \delta + f(\sete{y^{i-1}}{j}{\hat u_a}) - f(\sete{y^{i-1}}{j}{\opti{j}}) + \delta
	\end{flalign}
	Because $\delta_a\geq \delta_b$ then from the selection rule of $\delta_b$,
	\begin{flalign}\label{eq_20}
	\delta_a =  f(\x^i) - f(\x^{i-1})   \geq \delta_b \geq  f(\sete{y^{i-1}}{j}{c}) - f(\y^{i-1}) - \delta, \forall \underline{u}_j \leq c \leq \bar u_j. 
	\end{flalign}
	Setting $c = \opti{j}$ and substitite (\ref{eq_20}) into (\ref{eq_19}), one can get,
	\begin{flalign}
	0 > f(\sete{y^{i-1}}{j}{\hat u_a}) - f(y^{i-1}) +\delta = f(\y^{i+1}) - f(\y^{i})+\delta,
	\end{flalign}
	which contradicts with Lemma \ref{lemma_42}.
	
	\textcolor{emp_color}{\bf 2.2)} When $\delta_a < \delta_b$, it happens in the second change. $y^{i-1}_j =\hat u_b, x^i_j = \hat u_b > \opti{j}, OPT^i_j =\hat u_b, OPT^{i-1}_j = \opti{j}$. From submodularity,
		\begin{flalign}\label{eq_a1}
		f(OPT^i) + f(\sete{y^{i-1}}{j}{\opti{j}}) \geq f(OPT^{i-1}) + f(\sete{y^{i-1}}{j}{\hat u_b})
		\end{flalign}
		Suppose by virtue of contradiction that,
		\begin{flalign}\label{eq_b1}
		f(OPT^{i-1}) - f(OPT^{i}) > f(\x^i) - f(\x^{i-1}) + 2\delta .
		\end{flalign}
		Summing Equations \ref{eq_a1} and \ref{eq_b1} we get:
		\begin{flalign}
		0 > f(\x^i) - f(\x^{i-1})+\delta  + f(\sete{y^{i-1}}{j}{\hat u_b}) - f(\sete{y^{i-1}}{j}{\opti{j}}) +\delta. 
		\end{flalign}
	From Lemma \ref{lemma_42} we have $f(\x^i) - f(\x^{i-1}) +\delta \geq 0$, so $0 > f(\sete{y^{i-1}}{j}{\hat u_b}) - f(\sete{y^{i-1}}{j}{\opti{j}}) +\delta$, which contradicts with the selection rule of $\delta_b$. 
		
The case when  $\y$ is changed, $\x$ is kept unchanged
			 	is similar, the proof of which is omitted here.
\end{proof}

With Lemma \ref{lemma_43} at hand, one can prove Theorem  \ref{thm_double}: 
Taking a sum over $i$ from 1 to $2n$, one can get,
\begin{flalign}\notag
  f(OPT^0) - f(OPT^{2n}) &\leq f(\x^{2n}) - f(\x^{0}) + f(\y^{2n}) - f(\y^{0}) + 4n \delta  \\\notag
  & =  f(\x^{2n}) + f(\y^{2n})  - (f(\underline{\bu}) + f(\bar \bu)) + 4n \delta\\\notag
  &\leq  f(\x^{2n}) + f(\y^{2n}) + 4n \delta
\end{flalign}
Then it is easy to see that  $f(\x^{2n}) = f(\y^{2n}) \geq \frac{1}{3} f(\opt) - \frac{4n}{3}\delta$.

\section{Details of revenue maximization with continuous assignments}
\label{supp_revenue}
\subsection{More details about the model}
From the discussion in the main text, 
 $R_s(\x^i)$ should  be some non-negative, non-decreasing, submodular function, we set  $R_s(\x^i) := \allowbreak  \sqrt{\sum_{t: x^i_t \neq 0}x^i_t w_{st}}$, where 
 $w_{st}$ is the weight of edge connecting users $s$ and $t$.
The first part in R.H.S. of  \EQ  \ref{eq_re} models the revenue from users who have not received free assignments, while the second and third parts model the revenue from users who have gotten the free assignments.
We use $w_{tt}$ to denote the ``self-activation rate" of user $t$: Given
certain amount of free trail to user $t$, how probable is it that he/she will buy
after the trial. 
The intuition of modeling the second part in R.H.S. of  \EQ  \ref{eq_re}   is: Given the users more free assignments, they are more likely to buy the product after using it. 
Therefore, we model the expected revenue in this part by $\phi(x^i_t) =  w_{tt}x^i_t$; The intuition of modeling the third part in R.H.S. of  \EQ  \ref{eq_re}  is:  Giving the users more free assignments, the revenue could decrease, since   
the users use the product for free for a longer period.
 As a simple example,  the decrease in the revenue can be modeled as  $\gamma \sum_{t:x^i_t\neq 0} -x^i_t$.

\subsection{Proof of Lemma \ref{revenue}}

\begin{proof}
	
	First of all, we prove that $g(\x) : = \sum_{s: x_s =0} R_s(\x)$
	is a non-negative submodular function.
	
It is easy to see that $g(\x)$ is non-negative. 
To prove that $g(\x)$ is submodular, one just need,
\begin{flalign}\label{eq_f}
g(\a) + g(\b) \geq g(\a\vee \b) + g(\a\wedge \b), \quad  \forall \a, \b \in [0, \bar \bu].
\end{flalign}
Let $A:= \spt{\a}, B := \spt{\b}$, where $\spt{\x}:=\{i|x_i\neq 0 \}$ is  the  support of the vector $\x$.
First of all, because $R_s(\x)$ is non-decreasing,  and $\b\geq \a\wedge \b$, $\a\geq \a\wedge \b$, 
\begin{flalign}\label{eq_1}
\sum_{s\in A\backslash B} R_s(\b) + \sum_{s\in B\backslash A} R_s(\a) \geq \sum_{s\in A\backslash B} R_s(\a\wedge \b)  + \sum_{s\in B\backslash A} R_s(\a\wedge \b). 
\end{flalign}
By submodularity of $R_s(\x)$, and  summing over $s\in E \backslash(A\cup B)$,
\begin{flalign}\label{eq_2}
\sum_{s\in E \backslash(A\cup B)}R_s(\a) + \sum_{s\in E \backslash(A\cup B)}R_s(\b) \geq \sum_{s\in E \backslash(A\cup B)}R_s(\a\vee \b) + \sum_{s\in E \backslash(A\cup B)}R_s(\a\wedge \b).
\end{flalign}
Summing Equations  \ref{eq_1} and \ref{eq_2} one can get
\begin{flalign}\notag 
\sum_{s\in E \backslash A}R_s(\a) + \sum_{s\in E \backslash B}R_s(\b) \geq \sum_{s\in E \backslash(A\cup B)}R_s(\a\vee \b) + \sum_{s\in E \backslash(A\cap  B)}R_s(\a\wedge \b)
\end{flalign}
which is equivalent to \EQ \ref{eq_f}.

Then we prove that $h(\x):=\sum_{t: x_t \neq 0} \bar R_t(\x)$ is submodular. 
Because $\bar R_t(\x)$ is non-increasing, and $\a\leq \a\vee \b$, 
$\b \leq \a\vee \b$, 
\begin{flalign}\label{37}
	\sum_{t\in A\backslash B} \bar R_t(\a) + \sum_{t\in B\backslash A} \bar R_t(\b) \geq \sum_{t\in A\backslash B} \bar R_t(\a\vee \b) + \sum_{t\in B\backslash A} \bar R_t(\a\vee \b).
\end{flalign}
By submodularity of $\bar R_t(\x)$, and summing over $t\in A\cap  B$,
\begin{flalign}\label{38}
\sum_{t\in A\cap  B} \bar R_t(\a) + \sum_{t\in A\cap  B} \bar R_t(\b) \geq \sum_{t\in A\cap  B} \bar R_t(\a\vee \b) + \sum_{t\in A\cap  B} \bar R_t(\a\wedge \b).
\end{flalign}
Summing Equations \ref{37}, \ref{38} we get,
\begin{flalign*}
\sum_{t\in A} \bar R_t(\a) + \sum_{t\in  B} \bar R_t(\b) \geq \sum_{t\in A\cup  B} \bar R_t(\a\vee \b) + \sum_{t\in A\cap  B} \bar R_t(\a\wedge \b)
\end{flalign*}
which is equivalent to $h(\a)+h(\b)\geq h(\a\vee \b)+h(\a\wedge \b)$, $\forall \a, \b \in [0, \bar \bu]$, thus proving the submodularity of 
$h(\x)$.

Finally, because $f(\x)$ is the sum of two submodular functions and one 
modular function, so it is submodular.
\end{proof}

\subsection{Solving the 1-D subproblem when applying the \algname{DoubleGreedy} algorithm} \label{1d}
Suppose we are varying $x_j \in [0, \bar u_j]$ to maximize $f(\x)$, notice that this 1-D subproblem is non-smooth and discontinuous at  point $0$. 
First of all, let us leave $x_j = 0$ out, one can see that $f(\x)$ is 
concave and smooth along $\chara_j$ when $x_j\in (0, \bar u_j]$,
\begin{flalign}\notag 
\fracpartial{f(\x)}{x_j} =\alpha \sum_{s\neq j: x_s = 0} \frac{w_{sj}}{2\sqrt{\sum_{t: x_t \neq 0}x_t w_{st}}} - \gamma + \beta w_{jj}\\\notag  
\fracppartial{f(\x)}{x^2_j} = -\frac{1}{4}\alpha \sum_{s\neq j: x_s = 0} \frac{w_{sj}^2}{\left(\sqrt{\sum_{t: x_t \neq 0}x_t w_{st}}\right)^3} \leq 0.
\end{flalign}
Let $\bar f(z)$ be the univariate function when $x_j\in (0, \bar u_j]$, then we
extend the domain of $\bar f(z)$ to be $z\in [0, \bar u_j]$ as,
\begin{flalign}\notag 
\bar f(z) = \bar f(x_j):=  \alpha\!\! \sum_{s\neq j: x_s =0} R_s(\x)  + \beta\!\! \sum_{t\neq j: x_t \neq 0} \phi(x_t)+ \gamma \!\!\! \sum_{t\neq j: x_t \neq 0}\bar R_t(\x) + \beta \phi(x_j) + \gamma \bar R_j(\x).
\end{flalign}
One can see that $\bar f(z)$ is concave and smooth. 
Now to solve the 1-D subproblem, we can first of all solve the smooth concave
1-D maximization problem\footnote{It can be efficienlty solved by various methods, e.g., the bisection method or Newton method.}: $z^* := \argmax_{z\in [0, \bar u_j]} \bar f(z)$, then
compare $\bar f(z^*)$ with the function value at the discontinuous point $0$: $f(\sete{x}{j}{0})$, and return the point with the larger function value.


\section{More applications}\label{supp_more_app}

\paragraph{Maximum coverage.}
In the maximum coverage  problem, there are 
$n$ subsets $C_1,\cdots, C_n$ from the ground
set $V$.  One subset $C_i$ can be chosen
with ``confidence" level $x_i\in [0,1]$, the set of covered elements
when choosing subset $C_i$ with confidence $x_i$
can be modeled with the following monotone normalized covering function: $p_i: \R_+ \rightarrow 2^V, i=1,\cdots, n$. 
 The target is to choose 
subsets from $C_1,\cdots, C_n$ with confidence level 
to maximize the number of covered elements  $|\cup_{i=1}^n p_i(x_i)|$,
at the same time respecting the budget
constraint $\sum_i c_i x_i \leq b$ (where $c_i$ is the cost of choosing subset $C_i$).
This problem generalizes the classical maximum coverage problem. 
It is easy to see that the objective 
function is monotone submodular, and 
the constraint is a down-closed polytope.

\noindent\paragraph{Text summarization.}  Submodularity-based objective functions
for text summarization perform well in practice \citep{lin2010multi}. 
 Let $C$ to be the set of 
all concepts, and $E$ to be the set of all sentences. 
As a typical example,  the
concept-based summarization aims to find a subset $S$ of 
the sentences to maximize the total credit of concepts covered by 
$S$. \cite{soma2014optimal} discussed extending the submodular text summarization
model to the one that incorporates ``confidence" of a sentence,   which has
discrete value,  and modeling 
the objective to be a monotone submodular lattice function.
It is also natural to model the confidence level of  sentence $i$
to be a continuous value $x_i\in [0, 1]$. Let us use $p_i(x_i)$ to denote  the set of covered concepts 
when selecting sentence $i$ with confidence level $x_i$, it  can be  a monotone
covering function $p_i: \R_+ \rightarrow 2^C, \forall i\in E$. 
Then the objective function of the extended model is $f(\x) = \sum_{j\in \cup_i p_i(x_i) } c_j$, where $c_j\in \R_+$ is the credit of  concept $j$. It can be verified that  this objective is 
a monotone submodular continuous function.

\end{document}